\documentclass{article} 
\usepackage{iclr2027_conference,times}

\usepackage{amsmath,amsfonts,bm}

\def\1{\bm{1}}

\def\vzero{{\bm{0}}}
\def\vone{{\bm{1}}}

\def\va{{\bm{a}}}
\def\vb{{\bm{b}}}

\def\ve{{\bm{e}}}

\def\vq{{\bm{q}}}

\def\vv{{\bm{v}}}

\def\vx{{\bm{x}}}
\def\vy{{\bm{y}}}

\def\mA{{\bm{A}}}

\def\mD{{\bm{D}}}

\def\mK{{\bm{K}}}

\def\mQ{{\bm{Q}}}

\def\mT{{\bm{T}}}

\def\mV{{\bm{V}}}

\def\mY{{\bm{Y}}}

\DeclareMathAlphabet{\mathsfit}{\encodingdefault}{\sfdefault}{m}{sl}
\SetMathAlphabet{\mathsfit}{bold}{\encodingdefault}{\sfdefault}{bx}{n}

\def\sA{{\mathbb{A}}}
\def\sB{{\mathbb{B}}}

\def\sI{{\mathbb{I}}}

\def\sM{{\mathbb{M}}}
\def\sN{{\mathbb{N}}}

\def\sR{{\mathbb{R}}}

\newcommand{\R}{\mathbb{R}}

\newcommand{\softmax}{\mathrm{softmax}}

\DeclareMathOperator*{\argmax}{arg\,max}

\usepackage{amsmath}
\usepackage{mathtools}

\usepackage{booktabs}
\usepackage{tabularx}
\usepackage{array}
\usepackage{caption}
\usepackage{hyperref}
\usepackage{url}
\usepackage{cleveref}
\usepackage[store-sets-label]{keytheorems}

\newkeytheorem{theorem}[
  name=Theorem,
  parent=section,
  refname={theorem,theorems},
  Refname={Theorem,Theorems},
]

\newkeytheorem{lemma}[
  name=Lemma,
  sibling=theorem,
  refname={lemma,lemmas},
  Refname={Lemma,Lemmas},
]

\newkeytheorem{proposition}[
  name=Proposition,
  sibling=theorem,
  refname={proposition,propositions},
  Refname={Proposition,Propositions},
]

\newkeytheorem{corollary}[
  name=Corollary,
  sibling=theorem,
  refname={corollary,corollaries},
  Refname={Corollary,Corollaries},
]

\newkeytheorem{definition}[
  name=Definition,
  sibling=theorem,
  refname={definition,definitions},
  Refname={Definition,Definitions},
]

\newkeytheorem{assumption}[
  name=Assumption,
  sibling=theorem,
  refname={assumption,assumptions},
  Refname={Assumption,Assumptions},
]

\newkeytheorem{hypothesis}[
  name=Hypothesis,
  sibling=theorem,
  refname={hypothesis,hypotheses},
  Refname={Hypothesis,Hypotheses},
]

\newkeytheorem{example}[
  name=Example,
  sibling=theorem,
  refname={example,examples},
  Refname={Example,Examples},
]

\newkeytheorem{remark}[
  name=Remark,
  sibling=theorem,
  refname={remark,remarks},
  Refname={Remark,Remarks},
]

\usepackage[most]{tcolorbox}

\newcounter{boxedproblem}

\crefname{boxedproblem}{problem}{problem}
\Crefname{boxedproblem}{Problem}{Problem}

\newtcolorbox[use counter=boxedproblem]{boxedproblem}[3][]{%
  enhanced,
  breakable,
  title={Problem \theboxedproblem:\quad{#3}},
  label={#2},
  colback=black!2,
  colframe=black!30,
  coltitle=black,
  fonttitle=\bfseries,
  boxrule=0.5pt,
  arc=1.5pt,
  left=7pt,
  right=7pt,
  top=6pt,
  bottom=6pt,
  attach title to upper={\par\smallskip},
  #1
}

\newcommand{\seth}{\ensuremath{\mathsf{SETH}}}
\newcommand{\sat}{\ensuremath{\mathsf{SAT}}}

\newcommand{\AbsAtt}{\ensuremath{\mathsf{AbsAtt}}}
\newcommand{\RelAtt}{\ensuremath{\mathsf{RelAtt}}}
\newcommand{\AbsPreAtt}{\ensuremath{\mathsf{AbsPreAtt}}}
\newcommand{\RelPreAtt}{\ensuremath{\mathsf{RelPreAtt}}}
\newcommand{\AbsPreBatchAtt}{\ensuremath{\mathsf{AbsPreBatchAtt}}}
\newcommand{\RelPreBatchAtt}{\ensuremath{\mathsf{RelPreBatchAtt}}}
\newcommand{\SparseAtt}{\ensuremath{\mathsf{SparseAtt}}}
\newcommand{\bcp}{\ensuremath{\mathsf{CP}}}
\newcommand{\PreGapBCP}{\ensuremath{\mathsf{PreCP}}}

\DeclareMathOperator{\Att}{Att}
\DeclareMathOperator{\diag}{diag}
\DeclareMathOperator{\bigO}{\mathcal{O}}
\DeclareMathOperator{\littleO}{\textit{o}}
\DeclareMathOperator{\littleOmega}{\omega}
\DeclareMathOperator{\bigTheta}{\Theta}
\DeclareMathOperator{\bigOmega}{\Omega}
\DeclareMathOperator{\Schatten}{\mathcal{S}}
\DeclareMathOperator{\Exp}{Exp}

\title{Efficiently Approximating Attention Is Hard}

\author{%
\textbf{%
Lukas Haverbeck$^{\,\spadesuit}$\thanks{Corresponding author: \texttt{lukas.haverbeck@rwth-aachen.de}}
\quad
Carmen Amo Alonso$^{\,\clubsuit}$
\quad
Andres Felipe Posada-Moreno$^{\,\spadesuit}$%
}\\
\textbf{%
Sebastian Trimpe$^{\,\spadesuit}$
\quad
Marco Pavone$^{\,\clubsuit}$%
}\\[1em]
$\spadesuit$\; Institute for Data Science in Mechanical Engineering, RWTH Aachen University\\[0.5em]
$\clubsuit$\; Autonomous Systems Lab, Stanford University
}

\iclrpreprintcopy
\begin{document}

\maketitle

\vspace{2em}
\begin{abstract}
    Softmax attention is ubiquitous in modern machine learning, but its quadratic scaling with sequence length makes it costly. To reduce this cost, attention is often approximated with fast algorithms, which incur error but can still perform well in practice and on \emph{some} inputs.
    At the same time, the growing diversity of attention applications makes approximation guarantees that do not depend on particular input structure a compelling target.
    For such uniform guarantees over \emph{all} inputs, known runtime lower bounds rule out fast algorithms for near-exact attention, but leave open the practically important regime: \emph{is there an efficient algorithm with even a modest uniform approximation guarantee?}
    We answer this question negatively. Under standard complexity-theoretic assumptions, no truly subquadratic algorithm can approximate attention with any nontrivial additive or relative guarantee uniformly over all inputs.
    This impossibility holds in the mildest parameter regime for which known algorithms do not already achieve strong approximation guarantees in near-linear time, and extends to practically relevant relaxations: even after polynomial preprocessing of the KV cache, no efficient algorithm can obtain a nontrivial uniform approximation guarantee, or identify a small set of keys receiving substantial attention under sparsity.
    Overall, our results settle the computational limits of uniform attention approximation.
\end{abstract}
\newpage
\section{Introduction} \label{sec:introduction}
The Transformer architecture has become the de facto standard for the most challenging sequence modeling problems across various domains \citep{vaswani2017attention,devlin2019bert,dosovitskiy2020image,gulati2020conformer,rives2021biological,zhou2021informer}, owing largely to its ability to model complex relationships between different points in a sequence. The key architectural component enabling this is \emph{softmax attention}, which transforms a sequence of vectors based on pairwise comparisons. While attention makes Transformers powerful in practice, its standard implementation scales quadratically with the sequence length, making it the principal bottleneck of Transformer inference and often prohibitively expensive on long sequences \citep{beltagy2020longformer,dao2024flashattention}.

To address this problem, various methods aim to approximate attention with fast algorithms, which can work well empirically on typical inputs \citep{kitaev2020reformer,xiong2021nystromformer,tay2020sparse,roy2021routing} or admit formal guarantees under structural assumptions on the inputs \citep{vyas2020fast,chen2021scatterbrain,han2024hyperattention}. In practice, however, Transformers are deployed across increasingly diverse applications and input distributions, making it difficult to test approximation quality exhaustively or to verify that data-dependent assumptions are satisfied. That makes algorithms particularly attractive if they can provide approximation guarantees \emph{uniformly over all inputs}.

Remarkably strong approximation is already possible with uniform guarantees over sufficiently low-dimensional inputs with sufficiently small norms. In this regime, existing algorithms run in near-linear time, and guarantee error vanishing at polynomial or even super-polynomial rates in the sequence length \citep{alman2023fast,schroder2026wildcat}. Outside this regime, however, known lower bounds rule out such strong guarantees. More precisely, the best known lower bounds show that truly subquadratic algorithms, for some sequences of length \(n\), must incur error \(n^{-C}\), where \(C\gg 0\) is a large constant \citep{alman2023fast}. While these results rule out extreme accuracy with fast algorithms, they leave a substantial gap at the approximation accuracy most relevant in practice. In particular, existing lower bounds do not even rule out a linear-time algorithm that uniformly guarantees error at most, say, \(n^{-100}\). For all practical purposes, such high accuracy is effectively indistinguishable from exact computation, while empirical evidence suggests that Transformers remain robust even under much cruder approximations \citep{liu2024kivi,lin2022fqvit,michel2019sixteen}. This raises the question whether attention can be approximated efficiently once one tolerates a modest amount of error. We therefore ask:

\begin{center}
\textbf{\emph{Can attention be approximated efficiently with even a coarse uniform error guarantee?}}
\end{center}

We show that the answer is \textbf{\emph{no: such an algorithm does not exist}}. Concretely, we strengthen the proof of \citet{alman2023fast} to show that, under standard complexity-theoretic assumptions, no algorithm running in truly subquadratic\footnote{By ``truly'' subquadratic, we mean polynomially faster, that is, \(\bigO(n^{2-q})\) for some constant \(q>0\).} time can give \emph{any} nontrivial approximation guarantee uniformly over all inputs. Our result holds for both additive and relative error notions. Moreover, it applies under minimal assumptions on the data dimension and norms: even the slightest relaxation of these assumptions enters the setting in which existing algorithms supply near-linear runtime and vanishingly small error. Consequently, our new lower bound establishes a sharp transition of uniform attention approximation from possible at high accuracy to impossible at any accuracy.

Additionally, we study whether preprocessing or sparse attention structure can accelerate subsequent attention computations with uniform guarantees. Such approaches are common in practice, where one is often willing to incur the quadratic cost of initially processing the input, as well as additional preprocessing overhead, in exchange for faster inference on subsequent queries \citep{li2024snapkv,tang2024quest,chen2025magicpig,liu2025retrievalattention}. We show that, even in this relaxed setting, uniform approximation remains hard. Under the same mild assumptions as before, it is impossible to preprocess the input in polynomial time and subsequently approximate attention in truly sublinear time per token while guaranteeing either a nontrivial approximation of the attention output or, under significant sparsity, retrieval of a small set of keys capturing a constant fraction of the attention.

\paragraph{Contributions.} We provide new lower bounds that settle the limits of uniform attention approximation, closing the gap left by prior work and extending them to autoregressive decoding with preprocessing and sparsity. Beyond the regimes where uniform high-precision approximation is already efficiently possible, we show that approximate attention must exploit data-dependent structure.
\section{Related work} \label{sec:related-work}
\paragraph{Fast approximate attention.} The quadratic cost of the standard attention implementation has motivated extensive work on faster approximations. For instance, existing approaches exploit sparsity \citep{child2019sparse,zaheer2020bigbird,tay2020sparse,kitaev2020reformer,roy2021routing}, low-rank structure \citep{wang2020linformer,xiong2021nystromformer,chen2021scatterbrain}, or other data-dependent properties of the attention matrix \citep{zandieh2023kdeformer,han2024hyperattention,aliakbarpour2026support} for efficient attention approximation. These methods can substantially reduce computation while preserving model quality, but they demonstrate performance often only empirically for typical inputs or give guarantees that depend on structural properties of the input. In particular, most approaches do not provide uniform guarantees over all inputs even though such guarantees are appealing given the use of attention across domains as varied as language, vision, speech, biology, and time-series forecasting \citep{devlin2019bert,dosovitskiy2020image,gulati2020conformer,rives2021biological,zhou2021informer}. This raises \textbf{Question~(A):} \emph{is the dependence on input structure merely a limitation of existing algorithms or fundamentally unavoidable in practice?}

\paragraph{Uniform approximation guarantees.} Another line of work gives approximation guarantees uniformly over a broad class of inputs, typically satisfying a fixed norm bound without requiring additional instance-specific structure. For instance, \citet{choromanski2021performer} prove uniform convergence of a random-feature approximation when the query and key norms are bounded, although the number of features depends strongly on the radius and desired accuracy. \citet{alman2023fast} obtain a near-linear time algorithm with approximation error decaying at a rate of \(n^{-c}\) in the sequence length \(n\) for a configurable constant \(c>0\) when the entries of the key and query vectors have sufficiently small magnitude. More recently, \citet{schroder2026wildcat} achieve even super-polynomial error decay in the sequence length when keys and queries have sufficiently small norms. While these algorithms give approximation guarantees uniformly over all inputs, they have in common that they only result in fast algorithms for sufficiently low-dimensional inputs with sufficiently small norms. This raises \textbf{Question~(B):} \emph{are these restricted parameter regimes merely a limitation of existing algorithms or are they the broadest possible for uniform approximation guarantees?}

\paragraph{Runtime lower bounds.} Existing runtime lower bounds provide evidence that approximating attention is hard, but do not conclusively answer Questions~\textbf{(A)} or \textbf{(B)}. Early work by \citet{duman2023computational} establishes that near-exact attention computation requires near-quadratic runtime when the input dimension and norms are sufficiently large. More directly, \citet{alman2023fast} show that, outside the parameter regime in which their algorithm guarantees arbitrary polynomial error decay, there exists some constant \(C\gg 0\) such that no truly subquadratic algorithm can approximate attention with error at most \(n^{-C}\) uniformly over all input sequences of length \(n\). These results concern high-precision approximation, since \citet{alman2023fast} quantify \(C\) only existentially, while \citet{duman2023computational} only rule out fast algorithms with even faster-decaying additive error. More recent lower bounds by \citet[Thm.~5.4]{gupta2026subquadratic} rule out constant-error approximation, but only for inputs with extremely large norms, while their complementary results in more realistic regimes \citep[Thm.~5.7]{gupta2026subquadratic} again require extreme precision. In contrast, our new lower bound rules out \emph{any} nontrivial additive or relative error guarantee, thereby resolving Questions~\textbf{(A)} and \textbf{(B)}. At a technical level, our results strengthen the proof of \citet{alman2023fast} and therefore build on the lower bounds of \citet{rubinstein2018hardness} for the approximate closest pair problem.

\paragraph{Approximation with preprocessing and sparsity.} In practice, the runtime of Transformers on long sequences is often not dominated by the processing of the initial input but instead by the cost of subsequent queries during generation, since the former can be parallelized across the sequence whereas the latter cannot. Because of that, significant effort has gone into designing algorithms that speed up generation, often by preprocessing the sequence into some data structure that subsequently exploits sparsity to attend only over a small subset of the original input  \citep{li2024snapkv,tang2024quest,chen2025magicpig,liu2025retrievalattention}. These methods incur the full quadratic cost of attention before generation, as well as the added cost of preprocessing, in exchange for faster generation. Prior lower bounds leave open whether relaxing the computational problem in this way enables fast attention approximation with uniform guarantees during generation. We address this gap by extending our runtime lower bounds to the relaxed problem settings of approximate attention with preprocessing and retrieval under sparsity.
\newpage
\section{Nontrivial guarantees are impossible under mild assumptions} \label{sec:hardness}
In this section, we ask whether a fast algorithm can approximate attention uniformly over all inputs. After formally defining the problem in \Cref{subsec:hardness:problem-formulation}, we explain in \Cref{subsec:hardness:prior-work} why existing lower bounds do not resolve this question. Our main result in \Cref{subsec:hardness:result} then strengthens the lower-bound argument of \citet{alman2023fast} to rule out any nontrivial additive or relative approximation guarantee in truly subquadratic time.

\subsection{Problem formulation} \label{subsec:hardness:problem-formulation}
We consider the standard softmax attention operation over a sequence of input vectors. Attention allows each vector to selectively aggregate information from all other vectors in the sequence. To determine what information to aggregate, each element exposes a \emph{query}, a \emph{key}, and a \emph{value}. Each query is compared with all keys, and the resulting scores determine how the corresponding values are combined. Formally, this operation is defined as follows.
\begin{definition}[Attention] \label{def:attention}
    For matrices \(\mQ \in \R^{m \times d}\) and \(\mK,\mV\in\sR^{n\times d}\), let \(\mA = \exp(\mQ\mK^\top/d)\) and \(\mD = \diag(\mA\vone_n)\), where \(\exp\) is applied entrywise.
    Define \emph{attention} as
    \[
        \Att(\mQ,\mK,\mV)
        \coloneqq
        \mD^{-1}\mA\mV
        \in \R^{m \times d}.
    \]
\end{definition}
\begin{remark}
    For better comparability of our results with prior work, we follow \citet[Def.~1.1]{alman2023fast} in normalizing \(\mQ\mK^\top\) by \(d\).
    In practice, normalization often uses \(\sqrt{d}\) or a tunable parameter, instead.
    The precise normalization is inconsequential for our results, however, because it can always be absorbed into the keys and queries by suitable scaling.
\end{remark}

The straightforward implementation of attention computes the full attention matrix \(\mA\), comparing each of the \(m\) queries to each of the \(n\) keys before aggregating the values. In the usual setting where \(m=n\), this results in \(\bigO(n^2)\) operations. At the same time, both the input and output only have \(\bigO(nd)\) entries and the embedding dimension \(d\) is often much smaller than the sequence length \(n\). Consequently, one may hope for an algorithm that avoids explicitly forming all \(n^2\) comparisons and computes \(\Att(\mQ,\mK,\mV)\), at least approximately, in much less than \(\bigO(n^2)\) time. We formalize this problem below for different uniform approximation guarantees such an algorithm could give.

\begin{boxedproblem}{prb:approximate-attention}{Approximating attention}
    Fix parameters \(n,d \in \sN\) and \(B \ge 0\).
    Given matrices \(\mQ,\mK \in [-B,B]^{n \times d}\) and \(\mV \in [0,1]^{n \times d}\), the task is to compute a matrix \(\mT \in \sR^{n \times d}\).
    For this output, the problem
    \begin{align*}
        \AbsAtt(n,d,B,\eta_\mathrm{abs})
        &\qquad\text{requires}\qquad
        \|\mT-\Att(\mQ,\mK,\mV)\|_{\ell_\infty}
        \le \eta_\mathrm{abs},\quad\text{and}\\
        \RelAtt_p(n,d,B,\eta_\mathrm{rel})
        &\qquad\text{requires}\qquad
        \|\mT-\Att(\mQ,\mK,\mV)\|_{\Schatten_p} \le \eta_\mathrm{rel}\|\Att(\mQ,\mK,\mV)\|_{\Schatten_p}.
    \end{align*}
\end{boxedproblem}
We allow an algorithm to approximate attention with either an additive or a relative error guarantee. For additive error, we follow prior work in measuring error entrywise in \(\ell_\infty\) norm and restrict the values to \([0,1]\), since otherwise the error could be inflated arbitrarily by scaling the values. For relative error, we instead allow any Schatten norm \(\Schatten_p\), covering a broad range of natural relative and spectral guarantees not addressed by previous lower bounds. Under these conventions, additive error \(1/2\) and relative error \(1\) are trivial to achieve with no dependence on the input. A useful algorithm should therefore guarantee additive error below \(1/2\) or relative error below \(1\).

Our goal is to show that an algorithm with such nontrivial approximation guarantees cannot run in truly less than \(\bigO(n^2)\) time, under the mildest possible assumptions on \(d\) and \(B\). Concretely, we target \(d=\bigTheta(\log n)\) and \(B=\bigTheta(\sqrt{\log n})\), since making either parameter asymptotically smaller yields known algorithms with near-linear runtime and with vanishingly small error \citep{alman2023fast,schroder2026wildcat}. To rule out efficient algorithms at this boundary, we work under the strong exponential time hypothesis (\(\seth\)), which is a standard assumption in fine-grained complexity theory. In particular, essentially all prior lower bounds for attention are conditional on \(\seth\) \citep{alman2023fast,duman2023computational,gupta2026subquadratic}. 
\begin{hypothesis}[\seth, \citet{impagliazzo2001complexity}] \label{hyp:seth}
    For every \(\varepsilon>0\), there is an integer \(k\ge3\) such that \(k\)-\(\sat\) on formulas with \(n\) variables cannot be solved in \(\bigO(2^{(1-\varepsilon)n})\) time, not even by a randomized algorithm with bounded error probability.
\end{hypothesis}

\subsection{Existing lower bounds only rule out high-precision approximation} \label{subsec:hardness:prior-work}
Attention can already be efficiently approximated to extremely high accuracy when the data dimension and norms are sufficiently small.
Specifically, when either \(d=\bigO(\log n)\) and \(B=\littleO(\sqrt{\log n})\), or \(d=\littleO(\log n)\) and \(B=\bigO(\sqrt{\log n})\), there are algorithms with near-linear runtime solving \(\AbsAtt(n,d,B,\eta_\mathrm{abs})\) for any constant error guarantee \(\eta_\mathrm{abs}>0\).
Concretely, for any constant \(c>0\), \citet{alman2023fast} achieve error vanishing as \(n^{-c}\) with the sequence length in \(n^{1+\littleO(1)}\) time, and \citet{schroder2026wildcat} even achieve error vanishing as \(n^{-\littleOmega(1)}\) in \(n^{1+\littleO(1)}\) time.
Although asymptotic, these guarantees amount effectively to exact attention computation at the sequence lengths for which approximate attention is practically relevant.
Consequently, the mildest regime for which we do not already know how to solve \Cref{prb:approximate-attention} efficiently is at the boundary where \(d=\bigTheta(\log n)\) and \(B=\bigTheta(\sqrt{\log n})\).

At this boundary, existing runtime lower bounds show that efficient algorithms can no longer guarantee extreme accuracy.
Concretely, \citet{alman2023fast} show that \emph{some} polynomial error decay cannot be attained in truly subquadratic time unless \(\seth\) fails.
\begin{theorem}[\citet[Thm.~4.6]{alman2023fast}] \label{thm:alman-song-lower-bound}
    Fix \(q > 0\).
    Assuming \(\seth\), there exist constants \(D,B,C > 0\) such that \(\AbsAtt(n, D \log n, B \sqrt{\log n}, n^{-C})\) cannot be solved in \(\bigO(n^{2-q})\) time.
\end{theorem}
While this result rules out the extreme accuracy possible for small \(d\) and \(B\), it leaves open whether some of this accuracy can be exchanged for fast attention approximation in broader regimes.
In particular, \Cref{thm:alman-song-lower-bound} does not preclude a linear-time algorithm for \(\AbsAtt(n,\log n,\sqrt{\log n},0.49)\).

As summarized in \Cref{tab:lower-bounds}, other lower bounds do not resolve this question either. Closing the gap would require ruling out coarse approximation when \(d=\bigTheta(\log n)\) and \(B=\bigTheta(\sqrt{\log n})\). \citet{duman2023computational} rule out nontrivial entrywise-relative approximation, but only for super-polynomially small additive error and under more restrictive assumptions on \(d\) and \(B\). Likewise, \citet[Thm.~5.7]{gupta2026subquadratic} applies to \(d=\bigTheta(\log n)\) and \(B=\bigTheta(\sqrt{\log n})\), but again only shows that some polynomial decay is impossible, while \citet[Thm.~5.4]{gupta2026subquadratic} rules out constant-error approximation only in the practically unrealistic setting where query and key entries grow polynomially with \(n\).

Consequently, existing lower bounds leave open whether truly subquadratic algorithms can provide uniform approximation guarantees beyond the parameter regime of known algorithms with near-linear runtime and near-exact approximation. Closing this gap requires ruling out the entire nontrivial range of additive and relative errors already at \(d=\bigTheta(\log n)\) and \(B=\bigTheta(\sqrt{\log n})\).

\newcolumntype{Y}{>{\raggedright\arraybackslash}X}

\begin{table}[t]
    \centering
    \caption{Comparison of hardness results for approximating attention.
All results establish \(n^{2-\littleO(1)}\) runtime lower bounds under \(\seth\), but apply to different settings (head dimensions \(d\) and entry bounds \(B\)) and rule out different error guarantees.
For comparability, all parameters are standardized to \(\exp(\mQ\mK^\top/d)\) normalization.
``Error scale'' reports the error threshold below which hardness is established.
For the exact attention output \(\mY\in\R^{n\times d}\) and an approximation \(\mT\in\R^{n\times d}\), additive error means \(\|\mT-\mY\|_{\ell_\infty}\); entrywise-relative error means the smallest \(\eta\) such that \(|\mT_{ij}-\mY_{ij}| \le \eta |\mY_{ij}|\) for all \(i,j\); and norm-relative error means \(\|\mT-\mY\|_{\Schatten_p}/\|\mY\|_{\Schatten_p}\) for any Schatten norm \(\Schatten_p\).
Additive errors are comparable because all constructions use value entries in \([0,1]\).
All previous results either apply only under extreme temperature or rule out only extremely high-precision approximation.
We establish hardness in the mildest setting and for the coarsest error scales, ruling out \emph{every nontrivial approximation guarantee} under all three error notions.}
    \label{tab:lower-bounds}

    \renewcommand{\arraystretch}{1.2}
    \small
    \begin{tabular*}{\linewidth}{
        @{\extracolsep{\fill}}
        l c c c c c
        @{}
    }
        \toprule
        & \multicolumn{2}{c}{\textbf{Setting}}
        & \multicolumn{3}{c}{\textbf{Error scale}} \\
        \cmidrule(l){2-3}
        \cmidrule(l){4-6}
        \textbf{Result}
        & \footnotesize{\(d\)}
        & \footnotesize{\(B\)}
        & \footnotesize{Additive}
        & \footnotesize{Entrywise relative}
        & \footnotesize{Norm relative} \\
        \midrule

        \cite{gupta2026subquadratic} Thm.~5.4
        & \(2^{\bigTheta(\log^* n)}\)
        & \(n^{\bigTheta(1)}\)
        & \(\bigTheta(1)\)
        & \(\bigTheta(1)\) 
        & -- \\

        \cite{gupta2026subquadratic} Thm.~5.7
        & \(\Theta(\log n)\)
        & \(\Theta(\sqrt{\log n})\)
        & \(n^{-\bigTheta(1)}\)
        & \(n^{-\bigTheta(1)}\)  
        & -- \\

        \cite{duman2023computational}
        & \(\littleOmega(\log n)\)
        & \(\Theta(d^{3/2})\)
        & \(n^{-\bigOmega(d)}\)
        & \(1\)  
        & -- \\

        \cite{alman2023fast}
        & \(\bigTheta(\log n)\)
        & \(\bigTheta(\sqrt{\log n})\)
        & \(n^{-\Theta(1)}\)
        & \(n^{-\Theta(1)}\)  
        & -- \\

        \Cref{thm:constant-error-hardness} (ours)
        & \(\bigTheta(\log n)\)
        & \(\bigTheta(\sqrt{\log n})\)
        & \(\frac{1}{2}\)
        & \(1\)  
        & \(1\) \\

        \bottomrule
    \end{tabular*}
\end{table}

\subsection{Nontrivial guarantees are impossible under the mildest assumptions} \label{subsec:hardness:result}
We show that, once \(d=\bigTheta(\log n)\) and \(B=\bigTheta(\sqrt{\log n})\), no truly subquadratic algorithm can solve either \(\AbsAtt(n,d,B,\eta_\mathrm{abs})\) for any \(\eta_\mathrm{abs}<1/2\) or \(\RelAtt_p(n,d,B,\eta_\mathrm{rel})\) for any \(\eta_\mathrm{rel}<1\), unless \(\seth\) fails. To that end, we strengthen the proof of \citet{alman2023fast}, which draws on the connection between computing attention and finding close vectors.
\begin{definition}[store=def:bcp, name=Closest pair problem]
    For \(n,d\in\sN\) and \(\varepsilon>0\), the problem \(\bcp(n,d,\varepsilon)\) is the following.
    Given sets \(\{\va_1,\ldots,\va_n\}, \{\vb_1,\ldots,\vb_n\} \subseteq \{0,1\}^d\), find \(i^\star,j^\star \in [n]\) such that
    \[
        \|\va_{i^\star} - \vb_{j^\star}\|_0
        \le
        (1+\varepsilon) \min_{i,j\in[n]} \|\va_i - \vb_j\|_0.
    \]
\end{definition}
The central idea underlying the proof of \citet{alman2023fast} is that the attention matrix can be repurposed to encode pairwise distances between the two sets of vectors, while the subsequent value aggregation is used to recover a close pair. Hence, a fast algorithm for computing attention yields an essentially equally fast algorithm for finding an approximate closest pair. Finding such a pair, however among two sets of \(n\) binary vectors in dimension \(d=\bigO(\log n)\) cannot be done much more efficiently than checking all \(n^2\) possible pairs unless \(\seth\) fails, as \citet{rubinstein2018hardness} proves.
\begin{proposition}[store=pro:bcp-hardness, name=\citet[Thm.~1.1]{rubinstein2018hardness}]
    Fix \(q>0\).
    Assuming \(\seth\), there exist \(C>0\) and \(\varepsilon\in(0,1)\) such that \(\bcp(n,C\log n,\varepsilon)\) cannot be solved in \(\bigO(n^{2-q})\) time.
\end{proposition}
A truly subquadratic time algorithm for attention therefore cannot exist under \(\seth\). When only a fast \emph{approximate} algorithm for attention is available, however, the particular construction of \citet{alman2023fast} recovers such a pair only when the approximation error is extremely small. We therefore strengthen their construction to find a close pair reliably even when only a very coarse approximation of attention is available. Put together, any truly subquadratic algorithm for attention with a nontrivial approximation guarantee, already for \(d=\bigTheta(\log n)\) and \(B=\bigTheta(\sqrt{\log n})\), then implies a truly subquadratic algorithm for the closest pair problem, which would contradict \(\seth\) via \Cref{pro:bcp-hardness}. The following theorem formalizes this result. We sketch the main idea here and give the full proof in \Cref{apx:proofs}.
\begin{theorem}[store=thm:constant-error-hardness]
    Fix \(\eta_\mathrm{abs} \in [0, \frac{1}{2})\), \(\eta_\mathrm{rel} \in [0, 1)\), and \(q>0\).
    Assuming \(\seth\), there exist \(C_d,C_B > 0\) such that neither \(\AbsAtt(n,D,B,\eta_\mathrm{abs})\) nor \(\RelAtt_p(n,D,B,\eta_\mathrm{rel})\) can be solved in \(\bigO(n^{2-q})\) time, for \(D \coloneqq C_d \log n\), \(B \coloneqq C_B \sqrt{\log n}\), and any \(p \in [1,\infty]\).
\end{theorem}
\begin{proof}[Proof sketch.]
    Given two sets of binary vectors as an input to the closest pair problem, we construct from them query and key matrices so that the full \(n \times n\) attention matrix of all key--query comparisons already contains the information needed to find close pairs, with the largest attention scores corresponding to closest vectors. The difficulty is that an attention algorithm does not return these scores but only their normalized aggregation against the values, which can obscure the contribution of any individual close pair. Our main idea is to turn this normalization itself into a test whether a close pair with distance below a given threshold exists. To achieve this, we add dummy keys and augment the genuine queries and keys so that close pairs score above the dummies, while far pairs score below them. We then choose the values to measure the total attention paid to genuine keys. By scaling the logits appropriately, a single close genuine key can be made to outweigh all dummy keys whereas the combined contribution of all far genuine keys remains negligible, all while keeping the query and key entries bounded by \(B=\bigO(\sqrt{\log n})\). The resulting attention output then becomes an almost-binary indicator of whether a pair below distance \(t\) exists, allowing even any fixed nontrivial additive or relative approximation to distinguish the two cases. Repeating this test over logarithmically many thresholds recovers an approximate closest pair. Thus, a truly subquadratic algorithm for approximating attention would yield a truly subquadratic algorithm for the closest pair problem, contradicting \Cref{pro:bcp-hardness} unless \(\seth\) fails.
\end{proof}
\Cref{thm:constant-error-hardness} therefore settles the boundary of efficient attention approximation with uniform guarantees. At \(d=\bigTheta(\log n)\) and \(B=\bigTheta(\sqrt{\log n})\), even the weakest nontrivial guarantees require essentially quadratic work under standard complexity-theoretic assumptions. At the same time, relaxing either parameter to \(d=\littleO(\log n)\) or \(B=\littleO(\sqrt{\log n})\) admits algorithms with near-linear runtime and near-exact approximation \citep{alman2023fast,schroder2026wildcat}. Our results therefore identify a sharp transition of efficient attention approximation from possible at high accuracy to impossible at any nontrivial accuracy.
\newpage
\section{Approximation remains hard after KV cache preprocessing} \label{sec:preprocessing}
Transformer inference in practice is typically optimized for \emph{autoregressive generation}, where the model extends the original input sequence one token at a time. This differs from \emph{prefill}, where the initial input is processed all at once. Because autoregressive generation is inherently sequential, it is much harder to parallelize than prefill, making attention during generation often much more costly than during prefill. This has motivated substantial work on accelerating attention specifically during generation through preprocessing of the keys and values produced during prefill, collectively known as the \emph{KV cache}. In particular, such methods may spend additional computation during or after prefill to compress, index, or otherwise organize the KV cache to reduce the cost of subsequent attention queries \citep{li2024snapkv,tang2024quest,chen2025magicpig,liu2025retrievalattention}.

Allowing algorithms to preprocess the KV cache before approximating attention could, in principle, make uniform approximation substantially easier. Yet existing runtime lower bounds do not address this tradeoff between preprocessing and subsequent query time, despite its practical relevance. We therefore ask whether KV cache preprocessing enables fast attention approximation with nontrivial uniform guarantees that are impossible without it. To answer this question, \Cref{subsec:preprocessing:problem-formulation} adapts the computational problem of approximate attention to autoregressive generation after KV cache preprocessing. \Cref{subsec:preprocessing:result} then shows that, even with arbitrary polynomial-time preprocessing and under the same mild assumptions as before, nontrivial uniform guarantees remain impossible in every regime not already covered by fast high-accuracy approximation algorithms.

\subsection{Problem formulation} \label{subsec:preprocessing:problem-formulation}
To model approximate attention during autoregressive generation, we separate the approximation problem into two stages. In the preprocessing stage, the algorithm receives the KV cache and constructs from it an arbitrary data structure. Only afterward is a single query revealed, and the algorithm must compute its attention output using the preprocessed cache.
\begin{boxedproblem}{prb:preprocessed-attention}{Approximating attention with preprocessing}
    Fix parameters \(n,d \in \sN\) and \(B \ge 0\).
    Given matrices \(\mK \in [-B,B]^{n \times d}\) and \(\mV \in [0,1]^{n \times d}\), preprocess them. Then, given a query \(\mQ \in [-B,B]^{1 \times d}\), the task is to compute a matrix \(\mT \in \sR^{1 \times d}\).
    For this output, the problem
    \begin{align*}
        \AbsPreAtt(n,d,B,\eta_\mathrm{abs})
        &\quad\text{requires}\quad
        \|\mT - \Att(\mQ,\mK,\mV)\|_{\ell_\infty} \le \eta_\mathrm{abs},\quad\text{and}\\
        \RelPreAtt(n,d,B,\eta_\mathrm{rel})
        &\quad\text{requires}\quad
        \|\mT - \Att(\mQ,\mK,\mV)\|_{\Schatten_2} \le \eta_\mathrm{rel}\|\Att(\mQ,\mK,\mV)\|_{\Schatten_2}.
    \end{align*}
\end{boxedproblem}
The naive exact implementation for \Cref{prb:preprocessed-attention} simply stores the KV cache and subsequently scans the entire cache for each query, costing \(\bigO(nd)\) preprocessing time and \(\bigO(nd)\) decoding time. Furthermore, as before, additive error \(1/2\) and relative error \(1\) are trivial to guarantee without preprocessing. The relevant question is therefore whether nontrivial uniform approximation guarantees can be achieved in truly sublinear decoding time \(\bigO(n^{1-q})\) after potentially substantial preprocessing.

\subsection{Nontrivial guarantees are impossible under the mildest assumptions} \label{subsec:preprocessing:result}
To rule out fast approximation even after substantial preprocessing, we again exploit the connection between attention and the closest pair problem from \Cref{def:bcp}. In fact, \citet[Cor.~1.3]{rubinstein2018hardness} already shows that finding an approximate closest pair still requires checking essentially all pairs even when one of the two vector sets is revealed in advance and may be preprocessed in polynomial time. Incidentally, the reduction underlying \Cref{thm:constant-error-hardness} already mirrors this structure: one vector set determines \(\mK\) and \(\mV\), while the other is used only to form \(\mQ\). The same reduction therefore naturally extends to \Cref{prb:preprocessed-attention}. We defer the proof details to \Cref{apx:proofs} and state only the conclusion here. Under the same mild assumptions as before, arbitrary polynomial-time preprocessing does not enable efficient decoding with nontrivial uniform approximation guarantees.
\begin{theorem}[store=thm:preprocessing-hardness]
    Fix \(\eta_\mathrm{abs} \in [0, \frac{1}{2})\), \(\eta_\mathrm{rel} \in [0, 1)\), and \(k,q>0\).
    Assuming \(\seth\), there exist \(C_d,C_B > 0\) such that neither \(\AbsPreAtt(n,D,B,\eta_\mathrm{abs})\) nor \(\RelPreAtt(n,D,B,\eta_\mathrm{rel})\) can be solved with \(\bigO(n^k)\) preprocessing time and \(\bigO(n^{1-q})\) decoding time, for \(D \coloneqq C_d \log n\) and \(B \coloneqq C_B \sqrt{\log n}\).
\end{theorem}
\section{Approximation remains hard for sparse attention} \label{sec:sparsity}
Attention often exhibits sparse structure in practice, with a small number of tokens receiving most of the attention for any particular query. Many practical methods therefore exploit this structure to accelerate attention during generation by identifying the relevant tokens for each query and attending only over them \citep{li2024snapkv,tang2024quest,chen2025magicpig,liu2025retrievalattention}. Sparse attention methods therefore pursue an objective not captured by \Cref{prb:approximate-attention,prb:preprocessed-attention}: rather than approximating the full attention output, they aim to retain a small set of tokens receiving most of the attention. This can preserve accurate performance at substantially reduced cost if those tokens can be identified efficiently.

Identifying the important tokens could, in principle, be easier than approximating the full attention output, and neither prior work nor our preceding lower bounds rule out this possibility. This section therefore asks whether such retrieval is possible with uniform guarantees over all inputs. To that end, \Cref{subsec:sparsity:problem-formulation} formalizes retrieval under sparsity as a computational problem. \Cref{subsec:sparsity:result} then shows that, even after arbitrary polynomial-time preprocessing of the KV cache and under the same mild assumptions as before, no fast algorithm can identify a small set of tokens receiving substantial attention for a given query, even when its attention distribution is highly concentrated.

\subsection{Problem formulation} \label{subsec:sparsity:problem-formulation}
A sparse attention method maps a query to a small subset of relevant keys and subsequently evaluates attention only over this subset. To preserve the relevant information from the full computation, the selected keys should collectively receive a substantial fraction of the total attention. We therefore allow an algorithm to preprocess the keys into an arbitrary data structure and, once a query is revealed, ask it to return a small subset capturing a constant fraction of its attention. This objective is meaningful only when the attention distribution itself exhibits sparse structure. If all keys are identical, for example, attention is uniform and no small subset can receive much of the attention. We therefore restrict the problem to queries for which the attention distribution is actually sparse by promising that some key already receives a constant fraction of the attention. Returning any small set containing such a key then suffices.

\begin{boxedproblem}{prb:sparse-attention}{Sparse attention with preprocessing}
    Fix parameters \(n,d \in \sN\), \(B\ge0\), and \(\alpha\in(0,1)\).
    Given a matrix \(\mK \in [-B,B]^{n \times d}\), preprocess it.
    Then, given a query \(\vq \in [-B,B]^{d}\), the problem \(\SparseAtt(n,d,B,\alpha)\) asks for a subset \(\sI\subseteq[n]\) of the keys receiving attention
    \[
        \textstyle
        \sum_{i\in\sI} \softmax_i(\mK\vq/d) \ge \alpha
        \quad\text{promised that}\quad
        \softmax_i(\mK\vq/d) \ge \alpha
        \quad
        \text{for some }
        i\in[n].
    \]
\end{boxedproblem}
To solve this problem, a trivial algorithm could compute all attention scores and return a singleton containing the promised key, or it could return all keys, which always captures the full attention. Both approaches require \(\bigO(n)\) time, however, and therefore provide no speedup over simply computing attention over all keys. The relevant question is therefore whether \Cref{prb:sparse-attention} can be solved in truly sublinear time \(\bigO(n^{1-q})\) after preprocessing. In particular, we count the size of the set toward the runtime, requiring it to carry truly fewer than \(\bigO(n)\) keys.

\subsection{Efficiently retaining tokens receiving substantial attention is impossible} \label{subsec:sparsity:result}
We show that, under the same mild assumptions as before, no truly sublinear algorithm can retain any constant fraction of the attention after arbitrary polynomial-time preprocessing. While we still exploit the connection between attention and the closest pair problem from \Cref{def:bcp}, the reduction of \citet{alman2023fast} underlying \Cref{thm:constant-error-hardness,thm:preprocessing-hardness} does not extend directly to sparse attention. Its construction encodes close pairs through attention scores that are large relative to those of far pairs. Yet, a closest pair instance may contain many approximately closest pairs at similar distances. On such inputs, approximating the full attention output can reveal \emph{some} such pair even when the attention distribution is diffuse, violating the promise of \Cref{prb:sparse-attention}. To obtain hardness in the presence of sparsity, we therefore modify the reduction so that one approximately closest candidate dominates all others. We achieve this by randomly perturbing the encoded keys in a way that isolates one of the approximately closest keys to receive almost all of the attention. Any subset capturing a constant fraction of the attention must then contain this key and thereby reveal an approximate closest pair. We state the resulting lower bound below and sketch the reduction, deferring the full proof to \Cref{apx:proofs}.

\begin{theorem}[store=thm:sparse-hardness]
    Fix any \(\alpha \in (0,1)\) and \(k,q>0\).
    Assuming \(\seth\), there exist \(C_d,C_B>0\) such that \(\SparseAtt(n,D,B,\alpha)\) cannot be solved with \(\bigO(n^k)\) preprocessing time and \(\bigO(n^{1-q})\) decoding time, for \(D\coloneqq C_d\log n\) and \(B\coloneqq C_B\sqrt{\log n}\).
\end{theorem}
\begin{proof}[Proof sketch.]
    Suppose there were an algorithm for \Cref{prb:sparse-attention} with polynomial preprocessing time and truly sublinear decoding time. We use it to construct a bounded-error randomized algorithm for approximate nearest-neighbor search after polynomial preprocessing, contradicting \(\seth\) under the lower bound of \citet{rubinstein2018hardness}. As in \Cref{thm:constant-error-hardness}, we encode the preprocessed vectors as keys and the later query vector as an attention query, so that nearby vectors receive much larger scores than distant ones. The difficulty is that there may be many close keys, resulting in a diffuse attention distribution, whereas a sparse attention algorithm requires the attention distribution to be concentrated. Our main idea is to randomly perturb the key matrix so that one close key dominates all others with high probability. Concretely, we create a small number of copies of the key matrix and randomly shift each key by an exponentially distributed amount in a fixed direction. We truncate these shifts so that all entries remain bounded by \(B=\bigO(\sqrt{\log n})\) and far vectors continue to receive negligible attention. We show that, with high probability, at least one perturbed copy contains a key corresponding to an approximate closest neighbor of the query while receiving at least \(1-\frac{1}{n}>1-\alpha\) attention. For that copy, any subset capturing an \(\alpha\)-fraction of the attention must contain this key, thereby revealing an approximate closest neighbor. Running the hypothetical sparse attention algorithm on all copies still yields a truly sublinear bounded-error randomized algorithm for the closest-pair search problem and thus contradicting \(\seth\).
\end{proof}

This demonstrates that retrieval of relevant tokens under sparsity is fundamentally hard to guarantee uniformly over all inputs: in the mildest regime beyond the reach of highly accurate attention approximation with near-linear runtime, no algorithm can uniformly and efficiently identify a small set of tokens receiving substantial attention under future queries, even when attention is highly concentrated and even when allowing arbitrary polynomial-time preprocessing of the KV cache. Consequently, efficient sparse attention must exploit additional structure of the inputs beyond sparsity itself.
\section{Conclusion} \label{sec:conclusion}
Motivated by the computational cost of Transformers, we study whether attention can be approximated efficiently with uniform guarantees over all inputs. Under standard complexity-theoretic assumptions, we identify that the known tractable regimes mark the limit of what is possible: outside the regimes of existing algorithms with near-linear runtime and vanishingly small error, approximating attention with a nontrivial additive or relative guarantee uniformly over all inputs requires essentially quadratic time. Interestingly, this barrier persists during autoregressive generation after arbitrary polynomial-time preprocessing of the KV cache, and then even for retrieval under sparsity. While our results make substantial improvements in algorithms with uniform worst-case guarantees unlikely, they do not rule out algorithms that perform well empirically or obtain strong guarantees under structural assumptions on the inputs. In fact, our lower bounds substantiate the need for such assumptions: beyond the already tractable regimes, exploiting input structure is genuinely necessary to obtain meaningful guarantees for fast approximate attention.

\newpage
\section*{Acknowledgments}
Funded by the European Union.
This work has received funding from the  European High Performance Computing Joint Undertaking (JU) and from the  German Federal Ministry of Research, Technology and Space (BMFTR), the  Ministry of Culture and Science of North Rhine-Westphalia (MKW NRW) and  the Hessian Ministry of Science and Research, Arts and Culture (HMWK)  under grant agreement No \texttt{101250682}.
LH was supported by an RWTH Research Ambassador Scholarship.
CAA is supported by a Schmidt Science Fellowship.

\section*{AI Use Statement}
In the preparation of this paper, large language models were used to polish the authors' original writing. All scientific content was written and reviewed by the authors.

\bibliography{references}

@inproceedings{alman2023fast,
  author    = {Alman, Josh and Song, Zhao},
  title     = {Fast Attention Requires Bounded Entries},
  booktitle = {Advances in Neural Information Processing Systems},
  volume    = {36},
  pages     = {63117--63135},
  year      = {2023},
  doi       = {10.52202/075280-2755}
}

@article{impagliazzo2001complexity,
  author  = {Impagliazzo, Russell and Paturi, Ramamohan},
  title   = {On the Complexity of k-{SAT}},
  journal = {Journal of Computer and System Sciences},
  volume  = {62},
  number  = {2},
  pages   = {367--375},
  year    = {2001},
  doi     = {10.1006/jcss.2000.1727}
}

@inproceedings{rubinstein2018hardness,
  author    = {Rubinstein, Aviad},
  title     = {Hardness of approximate nearest neighbor search},
  booktitle = {Proceedings of the 50th Annual ACM SIGACT Symposium on Theory of Computing},
  pages     = {1260--1268},
  year      = {2018},
  doi       = {10.1145/3188745.3188916}
}

@inproceedings{duman2023computational,
  author    = {Duman Keles, Feyza and Wijewardena, Pruthuvi Mahesakya and Hegde, Chinmay},
  title     = {On The Computational Complexity of Self-Attention},
  booktitle = {Proceedings of The 34th International Conference on Algorithmic Learning Theory},
  pages     = {597--619},
  volume    = {201},
  series    = {Proceedings of Machine Learning Research},
  publisher = {PMLR},
  year      = {2023},
  url       = {https://proceedings.mlr.press/v201/duman-keles23a.html}
}

@inproceedings{gupta2026subquadratic,
 author = {Gupta, Shreya and Huang, Boyang and Saha, Barna and Xu, Yinzhan and Ye, Christopher},
 title = {Subquadratic Algorithms and Hardness for Attention with Any Temperature},
 booktitle = {International Conference on Learning Representations},
 year = {2026},
 url = {https://proceedings.iclr.cc/paper_files/paper/2026/file/8a01099096c85890b1d1aff3c6b4ea56-Paper-Conference.pdf},
}

@inproceedings{aliakbarpour2026support,
  author    = {Aliakbarpour, Maryam and Braverman, Vladimir and Yin, Junze and Zhang, Haochen},
  title     = {Support Basis: Fast Attention Beyond Bounded Entries},
  booktitle = {Proceedings of The 29th International Conference on Artificial Intelligence and Statistics},
  pages     = {325--333},
  volume    = {300},
  series    = {Proceedings of Machine Learning Research},
  publisher = {PMLR},
  year      = {2026},
  url       = {https://proceedings.mlr.press/v300/aliakbarpour26a.html}
}

@inproceedings{vaswani2017attention,
  author    = {Vaswani, Ashish and Shazeer, Noam and Parmar, Niki and Uszkoreit, Jakob and Jones, Llion and Gomez, Aidan N and Kaiser, {\L}ukasz and Polosukhin, Illia},
  title     = {Attention is All you Need},
  booktitle = {Advances in Neural Information Processing Systems},
  volume    = {30},
  pages     = {5998--6008},
  year      = {2017},
  url       = {https://proceedings.neurips.cc/paper_files/paper/2017/hash/3f5ee243547dee91fbd053c1c4a845aa-Abstract.html}
}

@inproceedings{devlin2019bert,
  author    = {Devlin, Jacob and Chang, Ming-Wei and Lee, Kenton and Toutanova, Kristina},
  title     = {{BERT}: Pre-training of Deep Bidirectional Transformers for Language Understanding},
  booktitle = {Proceedings of the 2019 Conference of the North American Chapter of the Association for Computational Linguistics: Human Language Technologies, Volume 1 (Long and Short Papers)},
  pages     = {4171--4186},
  year      = {2019},
  doi       = {10.18653/v1/N19-1423}
}

@inproceedings{dosovitskiy2020image,
  title={An Image is Worth 16x16 Words: Transformers for Image Recognition at Scale},
  author={Dosovitskiy, Alexey and Beyer, Lucas and Kolesnikov, Alexander and Weissenborn, Dirk and Zhai, Xiaohua and Unterthiner, Thomas and Dehghani, Mostafa and Minderer, Matthias and Heigold, Georg and Gelly, Sylvain and Uszkoreit, Jakob and Houlsby, Neil},
  booktitle={International Conference on Learning Representations},
  year={2021},
  url={https://openreview.net/forum?id=YicbFdNTTy}
}

@inproceedings{gulati2020conformer,
    title     = {Conformer: Convolution-augmented Transformer for Speech Recognition},
    author    = {Gulati, Anmol and Qin, James and Chiu, Chung-Cheng and Parmar, Niki and Zhang, Yu and Yu, Jiahui and Han, Wei and Wang, Shibo and Zhang, Zhengdong and Wu, Yonghui and Pang, Ruoming},
    booktitle = {Interspeech 2020},
    pages     = {5036--5040},
    year      = {2020},
    doi       = {10.21437/Interspeech.2020-3015}
}

@article{rives2021biological,
  author  = {Rives, Alexander and Meier, Joshua and Sercu, Tom and Goyal, Siddharth and Lin, Zeming and Liu, Jason and Guo, Demi and Ott, Myle and Zitnick, C. Lawrence and Ma, Jerry and Fergus, Rob},
  title   = {Biological structure and function emerge from scaling unsupervised learning to 250 million protein sequences},
  journal = {Proceedings of the National Academy of Sciences},
  volume  = {118},
  number  = {15},
  pages   = {e2016239118},
  year    = {2021},
  doi     = {10.1073/pnas.2016239118}
}

@article{zhou2021informer,
  author  = {Zhou, Haoyi and Zhang, Shanghang and Peng, Jieqi and Zhang, Shuai and Li, Jianxin and Xiong, Hui and Zhang, Wancai},
  title   = {Informer: Beyond Efficient Transformer for Long Sequence Time-Series Forecasting},
  journal = {Proceedings of the AAAI Conference on Artificial Intelligence},
  volume  = {35},
  number  = {12},
  pages   = {11106--11115},
  year    = {2021},
  doi     = {10.1609/aaai.v35i12.17325}
}

@inproceedings{kitaev2020reformer,
    title     = {Reformer: The Efficient Transformer},
    author    = {Kitaev, Nikita and Kaiser, {\L}ukasz and Levskaya, Anselm},
    booktitle = {International Conference on Learning Representations},
    year      = {2020},
    url       = {https://openreview.net/forum?id=rkgNKkHtvB}
}

@article{xiong2021nystromformer,
  author  = {Xiong, Yunyang and Zeng, Zhanpeng and Chakraborty, Rudrasis and Tan, Mingxing and Fung, Glenn and Li, Yin and Singh, Vikas},
  title   = {{Nystr{\"o}mformer}: A {Nystr{\"o}m}-based Algorithm for Approximating Self-Attention},
  journal = {Proceedings of the AAAI Conference on Artificial Intelligence},
  volume  = {35},
  number  = {16},
  pages   = {14138--14148},
  year    = {2021},
  doi     = {10.1609/aaai.v35i16.17664}
}

@inproceedings{tay2020sparse,
  author    = {Tay, Yi and Bahri, Dara and Yang, Liu and Metzler, Donald and Juan, Da-Cheng},
  title     = {Sparse {S}inkhorn Attention},
  booktitle = {Proceedings of the 37th International Conference on Machine Learning},
  pages     = {9438--9447},
  volume    = {119},
  series    = {Proceedings of Machine Learning Research},
  publisher = {PMLR},
  year      = {2020},
  url       = {https://proceedings.mlr.press/v119/tay20a.html}
}

@article{roy2021routing,
  author  = {Roy, Aurko and Saffar, Mohammad and Vaswani, Ashish and Grangier, David},
  title   = {Efficient Content-Based Sparse Attention with {Routing Transformers}},
  journal = {Transactions of the Association for Computational Linguistics},
  volume  = {9},
  pages   = {53--68},
  year    = {2021},
  doi     = {10.1162/tacl_a_00353}
}

@inproceedings{vyas2020fast,
    author    = {Vyas, Apoorv and Katharopoulos, Angelos and Fleuret, Fran{\c{c}}ois},
    title     = {Fast Transformers with Clustered Attention},
    booktitle = {Advances in Neural Information Processing Systems},
    volume    = {33},
    pages     = {21665--21674},
    year      = {2020},
    url       = {https://proceedings.neurips.cc/paper/2020/hash/f6a8dd1c954c8506aadc764cc32b895e-Abstract.html}
}

@inproceedings{chen2021scatterbrain,
    author    = {Chen, Beidi and Dao, Tri and Winsor, Eric and Song, Zhao and Rudra, Atri and R{\'e}, Christopher},
    title     = {Scatterbrain: Unifying Sparse and Low-rank Attention},
    booktitle = {Advances in Neural Information Processing Systems},
    volume    = {34},
    pages     = {17413--17426},
    year      = {2021},
    url       = {https://proceedings.neurips.cc/paper/2021/hash/9185f3ec501c674c7c788464a36e7fb3-Abstract.html}
}

@inproceedings{han2024hyperattention,
 author = {Han, Insu and Jayaram, Rajesh and Karbasi, Amin and Mirrokni, Vahab and Woodruff, David and Zandieh, Amir},
 booktitle = {International Conference on Learning Representations},
 title = {{HyperAttention}: Long-context Attention in Near-Linear Time},
 year = {2024},
 url = {https://proceedings.iclr.cc/paper_files/paper/2024/file/ab5aa940590399350401c57cdf52ce78-Paper-Conference.pdf}
}

@inproceedings{schroder2026wildcat,
  author    = {Schr{\"o}der, Tobias and Mackey, Lester},
  title     = {{WildCat}: Near-Linear Attention in Theory and Practice},
  booktitle = {Proceedings of the 43rd International Conference on Machine Learning},
  year      = {2026},
  url       = {https://openreview.net/forum?id=lfqyLp4hZm}
}

@inproceedings{liu2024kivi,
  author    = {Liu, Zirui and Yuan, Jiayi and Jin, Hongye and Zhong, Shaochen and Xu, Zhaozhuo and Braverman, Vladimir and Chen, Beidi and Hu, Xia},
  title     = {{KIVI}: A Tuning-Free Asymmetric 2bit Quantization for {KV} Cache},
  booktitle = {Proceedings of the 41st International Conference on Machine Learning},
  pages     = {32332--32344},
  volume    = {235},
  series    = {Proceedings of Machine Learning Research},
  publisher = {PMLR},
  year      = {2024},
  url       = {https://proceedings.mlr.press/v235/liu24bz.html}
}

@inproceedings{lin2022fqvit,
  author    = {Lin, Yang and Zhang, Tianyu and Sun, Peiqin and Li, Zheng and Zhou, Shuchang},
  title     = {{FQ-ViT}: Post-Training Quantization for Fully Quantized Vision Transformer},
  booktitle = {Proceedings of the Thirty-First International Joint Conference on Artificial Intelligence},
  pages     = {1173--1179},
  year      = {2022},
  doi       = {10.24963/ijcai.2022/164}
}

@inproceedings{michel2019sixteen,
  author    = {Michel, Paul and Levy, Omer and Neubig, Graham},
  title     = {Are Sixteen Heads Really Better than One?},
  booktitle = {Advances in Neural Information Processing Systems},
  volume    = {32},
  pages     = {14014--14024},
  year      = {2019},
  url       = {https://proceedings.neurips.cc/paper_files/paper/2019/hash/2c601ad9d2ff9bc8b282670cdd54f69f-Abstract.html}
}

@inproceedings{li2024snapkv,
  author    = {Li, Yuhong and Huang, Yingbing and Yang, Bowen and Venkitesh, Bharat and Locatelli, Acyr and Ye, Hanchen and Cai, Tianle and Lewis, Patrick and Chen, Deming},
  title     = {{SnapKV}: {LLM} Knows What You are Looking for Before Generation},
  booktitle = {Advances in Neural Information Processing Systems},
  volume    = {37},
  pages     = {22947--22970},
  year      = {2024},
  doi       = {10.52202/079017-0722}
}

@inproceedings{tang2024quest,
    title     = {{QUEST}: Query-Aware Sparsity for Efficient Long-Context {LLM} Inference},
    author    = {Tang, Jiaming and Zhao, Yilong and Zhu, Kan and
                 Xiao, Guangxuan and Kasikci, Baris and Han, Song},
    booktitle = {Proceedings of the 41st International Conference on Machine Learning},
    pages     = {47901--47911},
    year      = {2024},
    volume    = {235},
    series    = {Proceedings of Machine Learning Research},
    publisher = {PMLR},
    url       = {https://proceedings.mlr.press/v235/tang24l.html}
}

@inproceedings{chen2025magicpig,
 author = {Chen, Zhuoming and Sadhukhan, Ranajoy and Ye, Zihao and Zhou, Yang and Zhang, Jianyu and Nolte, Niklas and Tian, Yuandong and Douze, Matthijs and Bottou, Leon and Jia, Zhihao and Chen, Beidi},
 title = {{MagicPIG}: {LSH} Sampling for Efficient {LLM} Generation},
 booktitle = {International Conference on Learning Representations},
 year = {2025},
 url = {https://proceedings.iclr.cc/paper_files/paper/2025/file/6d50d824ae819d5a961c1d8edc15e833-Paper-Conference.pdf}
}

@inproceedings{liu2025retrievalattention,
  author    = {Liu, Di and Chen, Meng and Lu, Baotong and Jiang, Huiqiang and Han, Zhenhua and Zhang, Qianxi and Chen, Qi and Zhang, Chengruidong and Ding, Bailu and Zhang, Kai and Chen, Chen and Yang, Fan and Yang, Yuqing and Qiu, Lili},
  title     = {{RetrievalAttention}: Accelerating Long-Context {LLM} Inference via Vector Retrieval},
  booktitle = {Advances in Neural Information Processing Systems},
  volume    = {38},
  pages     = {54358--54385},
  year      = {2025},
  doi       = {10.52202/085713-1816}
}

@article{child2019sparse,
  author  = {Child, Rewon and Gray, Scott and Radford, Alec and Sutskever, Ilya},
  title   = {Generating Long Sequences with {Sparse Transformers}},
  journal = {arXiv preprint arXiv:1904.10509},
  year    = {2019},
  url     = {https://arxiv.org/abs/1904.10509}
}

@inproceedings{zaheer2020bigbird,
  author    = {Zaheer, Manzil and Guruganesh, Guru and Dubey, Kumar Avinava and Ainslie, Joshua and Alberti, Chris and Onta{\~n}{\'o}n, Santiago and Pham, Philip and Ravula, Anirudh and Wang, Qifan and Yang, Li and Ahmed, Amr},
  title     = {{Big Bird}: Transformers for Longer Sequences},
  booktitle = {Advances in Neural Information Processing Systems},
  volume    = {33},
  pages     = {17283--17297},
  year      = {2020},
  url       = {https://proceedings.neurips.cc/paper/2020/hash/c8512d142a2d849725f31a9a7a361ab9-Abstract.html}
}

@article{wang2020linformer,
  author  = {Wang, Sinong and Li, Belinda Z. and Khabsa, Madian and Fang, Han and Ma, Hao},
  title   = {{Linformer}: Self-Attention with Linear Complexity},
  journal = {arXiv preprint arXiv:2006.04768},
  year    = {2020},
  url     = {https://arxiv.org/abs/2006.04768}
}

@inproceedings{zandieh2023kdeformer,
  author    = {Zandieh, Amir and Han, Insu and Daliri, Majid and Karbasi, Amin},
  title     = {{KDE}former: Accelerating Transformers via Kernel Density Estimation},
  booktitle = {Proceedings of the 40th International Conference on Machine Learning},
  pages     = {40605--40623},
  volume    = {202},
  series    = {Proceedings of Machine Learning Research},
  publisher = {PMLR},
  year      = {2023},
  url       = {https://proceedings.mlr.press/v202/zandieh23a.html}
}

@inproceedings{choromanski2021performer,
  author    = {Choromanski, Krzysztof Marcin and Likhosherstov, Valerii and Dohan, David and Song, Xingyou and Gane, Andreea and Sarl{\'o}s, Tam{\'a}s and Hawkins, Peter and Davis, Jared Quincy and Mohiuddin, Afroz and Kaiser, {\L}ukasz and Belanger, David Benjamin and Colwell, Lucy J. and Weller, Adrian},
  title     = {Rethinking Attention with {Performers}},
  booktitle = {International Conference on Learning Representations},
  year      = {2021},
  url       = {https://openreview.net/forum?id=Ua6zuk0WRH}
}

@article{beltagy2020longformer,
  title={Longformer: The long-document transformer},
  author={Beltagy, Iz and Peters, Matthew E and Cohan, Arman},
  journal={arXiv preprint arXiv:2004.05150},
  year={2020},
  url={https://arxiv.org/abs/2004.05150}
}

@inproceedings{dao2024flashattention,
 author = {Dao, Tri},
 title = {{FlashAttention-2}: Faster Attention with Better Parallelism and Work Partitioning},
 booktitle = {International Conference on Learning Representations},
 year = {2024},
 url = {https://proceedings.iclr.cc/paper_files/paper/2024/file/98ed250b203d1ac6b24bbcf263e3d4a7-Paper-Conference.pdf},
}
\bibliographystyle{iclr2027_conference}

\newpage
\appendix
\crefalias{section}{appendix}
\section{Proofs} \label{apx:proofs}
This appendix contains the formal proofs of \Cref{thm:constant-error-hardness,thm:preprocessing-hardness,thm:sparse-hardness}.
\Cref{subsec:proofs:source-problem} presents a variant of the approximate closest pair problem, to which all our results reduce, and connects it to the lower bounds of \citet{rubinstein2018hardness}.
\Cref{subsec:proofs:encoding-closest-pairs-with-attention} shows how even a coarse approximation of attention can be used to identify approximate closest pairs, which \Cref{subsec:proofs:output-approximation} uses to prove \Cref{thm:constant-error-hardness,thm:preprocessing-hardness}.
Similarly, \Cref{subsec:proofs:encoding-closest-pairs-with-sparse-attention} shows how sparse attention retrieval can be used to identify approximate closest pairs, which \Cref{subsec:proofs:sparse-attention} uses to prove \Cref{thm:sparse-hardness}.

\subsection{Source problem} \label{subsec:proofs:source-problem}
We prove all our results by showing that a fast algorithm for attention with a given approximation guarantee would ultimately give a fast algorithm for the approximate closest pairs problem below, which \citet{rubinstein2018hardness} proves is impossible unless \(\seth\) fails.

\getkeytheorem{def:bcp}
\getkeytheorem{pro:bcp-hardness}

\begin{remark} \label{rem:bcp-randomized-hardness}
    Although \Cref{pro:bcp-hardness}, as stated, only concerns deterministic algorithms it also holds against bounded-error randomized algorithms under randomized \(\seth\). Indeed, both the standard reduction from \(k\)-\(\sat\) to Orthogonal Vectors and the subsequent reduction of \citet{rubinstein2018hardness} to \(\bcp\) are deterministic. Consequently, a bounded-error randomized algorithm for \(\bcp\) with truly subquadratic runtime would contradict \(\seth\) as stated in \Cref{hyp:seth}.
\end{remark}

For convenience, we define the decision variant of this problem with preprocessing.
We also follow \citet{alman2023fast} in working with inputs of balanced Hamming weights, that is, vectors in dimension \(2d\) that have Hamming weight \(d\).
To that end, define, for even \(d \in \sN\), the slice \(\sM_d \coloneqq \{ \vv \in \{0,1\}^d \mid \|\vv\|_0 = d/2\}\) of the Boolean cube with balanced Hamming weight.

\begin{definition}[Preprocessed Gap Closest Pair] \label{def:preprocessed-gap-bcp}
    For \(n,m,d\in\sN\) and \(\varepsilon>0\), the problem \(\PreGapBCP(n,m,d,\varepsilon)\) is the following.
    Given a threshold \(t \in [0,d]\) and set \(\sB=\{\vb_1,\ldots,\vb_n\} \subseteq \sM_d\), preprocess them.
    Then, given vectors \(\va_1,\ldots,\va_m \in \sM_d\), decide if
    \[
        \min_{i\in[m], \ j\in[n]} \|\va_i-\vb_j\|_0 \le t
        \qquad\text{or}\qquad
        \min_{i\in[m], \ j\in[n]} \|\va_i-\vb_j\|_0 > (1+\varepsilon)t,
    \]
    promised that one of the two holds.
\end{definition}

Solving this problem much more efficiently than checking all \(mn\) possible pairs remains hard even for bounded-error randomized algorithm with polynomial preprocessing and under the restriction to balanced inputs, as the following proposition shows.
The proof effectively combines the arguments behind \citet[Cor.~1.3]{rubinstein2018hardness} and \citet[Rem.~4.4]{alman2023fast} together with a standard boosting trick for randomization.

\begin{proposition}[Balanced gap version of \citet[Cor.~1.3]{rubinstein2018hardness}] \label{pro:preprocessed-gap-bcp-hardness}
    Fix \(k,q>0\).
    Suppose that \(\seth\) holds.
    Then, there exist \(C,\varepsilon>0\) such that \(\PreGapBCP(n,m(n),d,\varepsilon)\) cannot be solved, not even by a bounded-error randomized algorithm, with \(\bigO(n^k)\) preprocessing time and \(\bigO(m(n)n^{1-q})\) query time, for \(d \coloneqq C\log n\) and any function \(m: \sN\to\sN\) with \(m(n) \in [2n]\).
\end{proposition}
\begin{proof}
    Assume w.l.o.g. that \(k>1\) and \(q \le 1/4\).
    Let \(\gamma \coloneqq 1/(2k)\) and \(\bar q \coloneqq q\gamma/2\).
    By \Cref{pro:bcp-hardness}, there exist \(C>0\) and \(\varepsilon\in(0,1)\) such that \(\bcp(n,\frac{\gamma C}{2}\log n,\varepsilon)\) cannot be solved in \(\bigO(n^{2-\bar q})\) time.
    Suppose, for contradiction, that \(\PreGapBCP(n,m(n),d,\varepsilon)\) can be solved by a bounded-error randomized algorithm, for some function \(m: \sN\to\sN\) with \(m(n) \in [2n]\), with \(\bigO(n^k)\) preprocessing time and \(\bigO(m(n)n^{1-q})\) query time.
    We construct an algorithm for \(\bcp(n,\gamma d/2,\varepsilon)\) running in \(\bigO(n^{2-\bar q})\) time, contradicting \Cref{pro:bcp-hardness}.
    To that end, let sets \(\sA=\{\va_1,\ldots,\va_n\} \subseteq \{0,1\}^{\gamma d/2}\) and \(\sB=\{\vb_1,\ldots,\vb_n\} \subseteq \{0,1\}^{\gamma d/2}\) be inputs to \(\bcp(n,\frac{\gamma C}{2}\log n,\varepsilon)\).
    We replace them with the balanced sets \(\widehat\sA\coloneqq\{\widehat\va_1,\ldots,\widehat\va_n\}\subseteq\sM_{\gamma d}\) and \(\widehat\sB\coloneqq\{\widehat\vb_1,\ldots,\widehat\vb_n\}\subseteq\sM_{\gamma d}\), where \(\widehat\vx \coloneqq (\vx, \vone-\vx) \in \sM_{\gamma d}\), so that \(\|\va-\vb\|_0=\frac{1}{2} \|\widehat\va-\widehat\vb\|_0\).

    We now devise the algorithm.
    Let \(r \coloneqq n^\gamma\) and \(s \coloneqq m(r)\) so that \(\gamma d=C\log r\).
    Partition \(\widehat\sB\) into sets \(\widehat\sB_1,\ldots,\widehat\sB_{\lceil n/r \rceil}\), each of size at most \(r\).
    Then, for each threshold \(t \in \{0,\ldots,\gamma d\}\) in increasing order, do the following.
    Apply the assumed preprocessing algorithm for \(\PreGapBCP\) independently \(L\coloneqq C_L\log n\) times to every set \(\widehat\sB_j\) separately, padding it to \(r\) separate vectors by duplication when necessary.
    Additionally, partition \(\widehat\sA\) into sets \(\widehat\sA_1,\ldots,\widehat\sA_{\lceil n/s \rceil}\), each of size at most \(s\).
    On each preprocessed subset \(\widehat\sB_j\), independently run the corresponding \(L\) decoding algorithms on each query bulk \(\widehat\sA_i\), again padding them to \(s=m(r)\) vectors when necessary, and use their majority answer.
    That is, accept if at least half of the independent invocations of the decoding algorithm accept.

    Choosing \(C_L\) as a sufficiently large constant guarantees that, on every promised input, each majority answer is incorrect with probability at most \(n^{-4}\).
    Overall, we compute at most \(\bigO((\gamma d+1)(n/r)(n/s)) = \bigO(n^2 \log n)\) such answers.
    Hence, by a union bound, with probability at least
    \[
        1-\bigO(n^2\log n)n^{-4} \ge \frac{2}{3},
    \]
    all majority answers on promised inputs are correct simultaneously.
    Condition on that event and consider the smallest \(t\) for which the majority answer accepts on some \(\widehat\sB_j\) and \(\widehat\sA_i\).
    Then,
    \[
        \min_{\widehat\va\in\widehat\sA_i, \ \widehat\vb\in\widehat\sB_j} \|\widehat\va-\widehat\vb\|_0
        \le
        (1+\varepsilon) t
        \le
        (1+\varepsilon) \min_{\widehat\va\in\widehat\sA, \widehat\vb\in\widehat\sB} \|\widehat\va-\widehat\vb\|_0,
    \]
    where the first inequality uses that \(\min_{\widehat\va\in\widehat\sA_i, \widehat\vb\in\widehat\sB_j} \|\widehat\va-\widehat\vb\|_0 > (1+\varepsilon)t\) would force the decoding algorithm to reject by the problem definition, and the second inequality uses that the algorithm must accept once \(t=\min_{\widehat\va\in\widehat\sA, \widehat\vb\in\widehat\sB} \|\widehat\va-\widehat\vb\|_0\).
    Consequently, computing \(\arg\min_{\widehat\va\in\widehat\sA_i, \widehat\vb\in\widehat\sB_j}\) gives a pair \((\va,\vb)\) for the original input with
    \[
        \|\va-\vb\|_0
        =
        \frac{1}{2} \|\widehat\va-\widehat\vb\|_0
        \le
        \frac{1}{2} (1+\varepsilon) \min_{\widehat\va'\in\widehat\sA, \widehat\vb'\in\widehat\sB} \|\widehat\va'-\widehat\vb'\|_0
        =
        (1+\varepsilon) \min_{\va'\in\sA, \vb'\in\sB} \|\va'-\vb'\|_0
    \]
    satisfying the output condition of \(\bcp\).
    Since we only need to search the small subsets \(\widehat\sA_i\) and \(\widehat\sB_j\), such a pair can be computed by brute force in \(\bigO(rsd) \le \bigO(r^2d) \le \bigO(n\log n)\) time.

    Consequently, we obtain an algorithm with bounded error for \(\bcp(n,\gamma d/2,\varepsilon)\).
    For each threshold, the preprocessing takes \(\bigO(L\frac{r^kn}{r}) \le \bigO(n^{3/2}\log n)\) time and the decoding takes \(\bigO(L\frac{n}{r}\frac{n}{s}sr^{1-q}) \le \bigO(n^{2-q\gamma} \log n)\) time.
    Repeating for all \(\gamma d+1\) thresholds adds a factor \(\bigO(d) = \bigO(\log n)\).
    Thus, the algorithm runs in \(\bigO(n^{2-q\gamma/2})\) time, contradicting \Cref{pro:bcp-hardness} together with \Cref{rem:bcp-randomized-hardness}.
\end{proof}

We also use the fact that the problem is easy to solve when the decision threshold \(t\) is small.
To prove this, we apply the argument of \citet{alman2023fast}[Lem.~B.1, Case~1] to our particular problem formulation to give a brute-force algorithm searching the immediate vicinity of every preprocessed vector, resulting in \(\bigO(mn^{1-q})\) runtime when the search space is sufficiently small.

\begin{lemma} \label{lem:small-t-bcp}
    Fix \(C_d>0\) and \(q\in(0,1)\).
    There exists a constant \(c\in(0,C_d/2)\) such that \(\PreGapBCP(n,m,C_d\log n,\varepsilon)\) can be solved with \(\bigO(nd)\) preprocessing time and \(\bigO(mn^{1-q})\) decoding time whenever the input threshold \(t\) satisfies \(t\le c\log n\).
\end{lemma}
\begin{proof}
    Let \(d\coloneqq C_d\log n\).
    Choose any \(c\in(0,C_d/2)\) such that \(c \log\left( \frac{eC_d}{c} \right) < 1-q\).
    Such a choice is always possible since \(\lim_{c \searrow 0} c \log\left( \frac{eC_d}{c} \right) = 0\).
    Now, let \(\sB\coloneqq\{\vb_1,\ldots,\vb_n\}\subseteq\sM_d\) be the input set to \(\PreGapBCP(n,m,d,\varepsilon)\) and \(\va_1,\ldots,\va_m\in\sM_d\) the query revealed after preprocessing.
    We give a brute-force algorithm for deciding this input, following the Hamming-ball enumeration of \citet[Lem.~B.1, Case 1]{alman2023fast}.
    Preprocess \(\sB\) by storing it in a trie.
    This takes \(\bigO(nd)\) time.
    Then, enumerate, for all \(i\in[m]\) separately, all vectors \(\vb\in\{0,1\}^d\) satisfying \(\|\va_i-\vb\|_0\le t\), and accept if any such \(\vb\) belongs to \(\sB\).
    This correctly decides the instance: if some \(\va_i\) has a close partner in \(\sB\), exhaustive enumeration finds it, and if no \(\va_i\) has such a partner, no partner is found.
    Since \(t\le c\log n<d/2\), the decoding time is
    \begin{align*}
        \bigO\left(
            md\sum_{k=0}^{\lfloor t\rfloor}
            \binom{d}{k}
        \right)
        \le
        \bigO\left(
            md\left(\frac{ed}{t}\right)^t
        \right)
        &\le
        \bigO\left(
            m\log n\left(\frac{eC_d\log n}{c\log n}\right)^{c\log n}
        \right)
        \\&\le
        \bigO\left(
            mn^{c\log(eC_d/c)+\littleO(1)}
        \right)
        \le
        \bigO\left(mn^{1-q}\right).
        \tag*{\qedhere}
    \end{align*}
\end{proof}

\subsection{Finding close pairs from a coarse attention approximation} \label{subsec:proofs:encoding-closest-pairs-with-attention}
Following the general strategy of \citet{alman2023fast}, we encode a closest pair instance as an attention instance so that approximating the attention output reveals an approximate closest pair. Deviating from their construction, however, we ensure that recovering such a pair does not require a high-precision attention approximation. To achieve this, we shift the attention scores relative to a collection of dummy keys so that the contribution of a close pair dominates the output, while the combined contribution of all far pairs remains negligible. Importantly, the keys and values depend only on one of the two closest pair sets, whereas the queries depend only on the other. The construction is therefore compatible with preprocessing the keys and values before the query vectors are revealed. The following lemma formalizes these properties.

\begin{lemma} \label{lem:gadget}
    Fix \(\varepsilon,c,C_d > 0\).
    There exists \(C_B>0\) such that the following holds for all sufficiently large \(n \in \sN\) and \(N \coloneqq 2n\), \(d \coloneqq 2C_d \log n\), \(D \coloneqq 4C_d \log N\), and \(B \coloneqq C_B \sqrt{\log N}\).
    
    Choose \(t \ge c \log n\) with \(d \ge (1+\varepsilon)t\), a set of vectors \(\sB = \{\vb_1,\ldots,\vb_n\} \subseteq \sM_d\), and a sequence \(\va_1,\ldots,\va_m \in \sM_d\).
    One can construct matrices \(\mK \in [-B,B]^{N \times D}\) and \(\mV \in [0,1]^{N \times D}\) independent of \(\va_1,\ldots,\va_m\) and in \(\bigO(nD)\) time, and another matrix \(\mQ \in [-B,B]^{m \times D}\) in \(\bigO(mD)\) time such that \(\Att(\mQ,\mK,\mV) = \vy \ve_1^\top\) for some \(\vy \in \R^m\) satisfying, for every \(i \in [m]\),
    \begin{align*}
        y_i &\ge \frac{n}{n+1} \qquad \text{if} \quad \min_{j \in [n]} \|\va_i-\vb_j\|_0 \le t, \quad\text{and} \\
        y_i &\le \frac{1}{n+1} \qquad \text{if} \quad \min_{j \in [n]} \|\va_i-\vb_j\|_0 \ge (1+\varepsilon)t.
    \end{align*}
\end{lemma}
\begin{proof}
    Define
    \[
        \rho \coloneqq \sqrt{D/(2d)}
        ,\qquad
        \beta
        \coloneqq
        \frac{12d\log n}{\varepsilon t}
        ,\qquad \text{and} \qquad
        \mu
        \coloneqq
        \frac{1}{2} - \left(\frac{1}{2} + \frac{\varepsilon}{3}\right) \frac{t}{d}
    \]
    and note that
    \[
        0 \le \frac{\varepsilon}{6(1+\varepsilon)} \le \mu \le \frac{1}{2}
    \]
    since \(t/d \le 1/(1+\varepsilon)\) by assumption.
    
    We construct the key and value matrices as
    \[
        \mK
        \coloneqq
        \rho\sqrt{\beta}
        \begin{bmatrix}
            \vb_1^\top & -\mu\vone_d^\top & \vzero_{D-2d}^\top \\
            \vdots & \vdots & \vdots \\
            \vb_n^\top & -\mu\vone_d^\top & \vzero_{D-2d}^\top \\
            \vzero_d^\top & \vzero_d^\top & \vzero_{D-2d}^\top \\
            \vdots & \vdots & \vdots \\
            \vzero_d^\top & \vzero_d^\top & \vzero_{D-2d}^\top
        \end{bmatrix}
        \in \R^{N \times D}
        \qquad \text{and} \qquad
        \mV
        \coloneqq
        \begin{bmatrix}
            1 & 0 & \cdots & 0 \\
            \vdots & \vdots & \ddots & \vdots \\
            1 & 0 & \cdots & 0 \\
            0 & 0 & \cdots & 0 \\
            \vdots & \vdots & \ddots & \vdots \\
            0 & 0 & \cdots & 0 \\
        \end{bmatrix}
        \in \R^{N \times D}
    \]
    and the query matrix as
    \[
        \mQ
        \coloneqq
        \rho\sqrt{\beta}
        \begin{bmatrix}
            \va_1^\top & \vone_d^\top & \vzero_{D-2d}^\top \\
            \vdots & \vdots & \vdots \\
            \va_m^\top & \vone_d^\top & \vzero_{D-2d}^\top \\
        \end{bmatrix}
        \in \R^{m \times D}.
    \]
    
    Next, we choose \(C_B\) so that the constructed matrices satisfy the required entry bounds.
    Since \(t \ge c\log n\), \(d=2C_d\log n\), and \(\rho^2=\log(2n)/\log n\le 2\),
    \[
        \rho^2\beta
        =
        \frac{\log N}{\log n} \frac{12 d \log n}{\varepsilon t}
        \le
        \frac{24C_d \log N}{\varepsilon c}
    \]
    Hence, there exists a constant \(C_B>0\), depending only on \(C_d\), \(\varepsilon\), and \(c\), such that \(\rho\sqrt{\beta} \le C_B\sqrt{\log N}\).
    Since \(0\le\mu\le 1/2\), it follows that \(\|\mQ\|_{\ell_\infty}\le B\) and \(\|\mK\|_{\ell_\infty}\le B\).

    We now show that the attention output separates close from far pairs as claimed.
    To that end, let \(\mY \coloneqq \Att(\mQ,\mK,\mV)\).
    For \(i \in [m]\) and \(j\in[n]\), the logit between the \(i\)-th query and the \(j\)-th non-dummy key is
    \begin{align}
        \frac{\langle \mQ_i,\mK_j\rangle}{D}
        &=
        \frac{\rho^2\beta}{D}
        \left(\langle a_i,b_j\rangle-\mu d\right) \notag
        =
        \frac{\beta}{2d}
        \left(\langle a_i,b_j\rangle-\mu d\right) \notag
        \\&=
        \frac{\beta}{2d} \left( \frac{\|a_i\|_2^2 + \|b_j\|_2^2 - \|a_i-b_j\|_2^2}{2} - \mu d \right) \notag
        \\&=
        \frac{\beta}{2d} \left( \frac{d-\|a_i-b_j\|_0}{2} - \mu d \right) \notag
        \\&=
        \frac{\beta}{4} \left( 1-\frac{\|a_i-b_j\|_0}{d} \right) - \frac{\beta\mu}{2}, \label{eq:gadget:logit}
    \end{align}
    where the third equality uses that \(a_i\) and \(b_j\) both have Hamming weight \(d/2\).
    For \(i\in[m]\), let
    \begin{equation}
        s_i
        \coloneqq
        \sum_{j=1}^n
        \exp\left(
            \frac{\langle \mQ_i,\mK_j\rangle}{D}
        \right)
    \end{equation}
    be the total score assigned to   non-dummy keys.
    By the definition of \(\mV\), the \(i\)-th attention output is
    \begin{equation} \label{eq:reduction:output}
        \mY_i = \frac{s_i}{s_i+n}\ve_1^\top,
    \end{equation}
    since the \(n\) dummy keys are \(0\) so that each of them has logit \(0\) and contributes score \(1\) and only the first column of \(\mV\) is nonzero.
    Now, fix \(i \in [m]\) and distinguish the following two cases from the claim.
    Suppose first that \(\min_{j\in[n]}\|a_i-b_j\|_0 \le t\).
    Then, \Cref{eq:gadget:logit} and the definition of \(\mu\) imply
    \[
        \frac{\langle \mQ_i,\mK_j \rangle}{D}
        \ge
        \frac{\beta}{4} \left(1-\frac{t}{d}\right) - \frac{\beta\mu}{2}
        =
        \frac{\beta \varepsilon t}{6d}
        =
        2\log(n)
    \]
    for some \(j \in [n]\).
    Consequently, \(s_i \ge n^2\) and \Cref{eq:reduction:output} gives \(Y_{i,1} \ge n/(n+1)\).
    Now, suppose instead that \(\min_{j\in[n]}\|a_i-b_j\|_0\ge(1+\varepsilon)t\).
    Then, \Cref{eq:gadget:logit} and the definition of \(\mu\) imply
    \[
        \frac{\langle \mQ_i,\mK_j\rangle}{D}
        \le
        \frac{\beta}{4} \left( 1 - \frac{(1 + \varepsilon)t}{d} \right) - \frac{\beta\mu}{2}
        =
        -\frac{\beta \varepsilon t}{12d}
        =
        -\log n
    \]
    for all \(j \in [n]\).
    Consequently, \(s_i\le 1\), and \Cref{eq:reduction:output} gives \(Y_{i,1} \le 1/(n+1)\) for all \(i \in [m]\).

    Finally, \(\mK\) and \(\mV\) can be constructed in \(\bigO(ND)\) time and \(\mQ\) in \(\bigO(mD)\) time by copying and padding the respective inputs.
    That proves the claim.
\end{proof}

\newpage
\subsection[Proofs of Theorems~\ref*{thm:constant-error-hardness} and \ref*{thm:preprocessing-hardness}]{Proofs of \Cref{thm:constant-error-hardness,thm:preprocessing-hardness}} \label{subsec:proofs:output-approximation}
\Cref{thm:constant-error-hardness,thm:preprocessing-hardness} can now be proved by using \Cref{lem:gadget} to encode an input to the closest pair problem from \Cref{def:preprocessed-gap-bcp} as an attention input whose output gives an almost-binary indicator of whether an approximate close pair exists.
Consequently, a fast algorithm with even a coarse approximation would give a fast algorithm for deciding the closest pair problem, which would contradict \Cref{pro:preprocessed-gap-bcp-hardness}.
While the proof is essentially the same for both theorems, the underlying problem formulations differ slightly.
For compactness, we therefore formulate below a generalized problem formulation including both \Cref{prb:approximate-attention,prb:preprocessed-attention} as special cases.
We then prove that this generalized problem cannot be solved much more efficiently than by the trivial algorithm for attention, from which \Cref{thm:constant-error-hardness,thm:preprocessing-hardness} follow immediately.

\begin{boxedproblem}{prb:preprocessed-batch-attention}{Approximating attention with preprocessing (batch variant)}
    Fix parameters \(n,m,d \in \sN\) and \(B\ge0\).
    Given matrices \(\mK \in [-B,B]^{n \times d}\) and \(\mV \in [0,1]^{n \times d}\), preprocess them.
    Then, given a query \(\mQ \in [-B,B]^{m \times d}\), the task is to compute a matrix \(\mT \in \sR^{m \times d}\).
    For this output and the exact target \(\mY \coloneqq \Att(\mQ,\mK,\mV)\), the problem
    \begin{align*}
        \AbsPreBatchAtt(n,m,d,B,\eta_\mathrm{abs})
        &\qquad\text{requires}\qquad
        \|\mT - \mY\|_{\ell_\infty} \le \eta_\mathrm{abs},\quad\text{and}\\
        \RelPreBatchAtt_p(n,m,d,B,\eta_\mathrm{rel})
        &\qquad\text{requires}\qquad
        \|\mT - \mY\|_{\Schatten_p} \le  \eta_\mathrm{rel}\|\mY\|_{\Schatten_p}.
    \end{align*}
\end{boxedproblem}

\begin{proposition} \label{pro:preprocessed-batched-attention-hardness}
    Fix \(\eta_\mathrm{abs} \in [0,1/2)\), \(\eta_\mathrm{rel} \in [0,1)\) and \(k,q>0\).
    Suppose that \(\seth\) holds.
    Then, there exist \(C_d,C_B,\varepsilon>0\) such that neither \(\AbsPreBatchAtt(n,m(n),d,B,\eta_\mathrm{abs})\) nor \(\RelPreBatchAtt_p(n,m(n),d,B,\eta_\mathrm{rel})\) can be solved with \(\bigO(n^k)\) preprocessing time and \(\bigO(m(n)n^{1-q})\) query time, for \(d \coloneqq C\log n\), \(B \coloneqq C_B\sqrt{\log n}\), any \(p \in [1,\infty]\), and any function \(m: \sN\to\sN\) with \(m(n) \in [2n]\).
\end{proposition}
\begin{proof}
    Define shorthand notation \(N \coloneqq 2n\) and \(m \coloneqq m(n)\).
    Assume w.l.o.g. that \(k>1\), \(q\in(0,1/2)\), and that \(n\) is sufficiently large.
    We choose \(C_d,\varepsilon>0\) as the constant from \Cref{pro:preprocessed-gap-bcp-hardness} so that \(\PreGapBCP(n,m,d,\varepsilon)\) cannot be solved with \(\bigO(n^k)\) preprocessing time and \(\bigO(mn^{1-q})\) query time for \(d\coloneqq \frac{C_d}{2}\log n\).
    Suppose, for contradiction, that either \(\AbsPreBatchAtt(n,m,D,B,\eta_\mathrm{abs})\) or \(\RelPreBatchAtt_p(n,m,D,B,\eta_\mathrm{rel})\) can be solved with \(\bigO(n^k)\) preprocessing time and \(\bigO(mn^{1-q})\) decoding time.

    We construct an algorithm for \(\PreGapBCP(n,m(n),d,\varepsilon)\) with \(\bigO(n^k)\) preprocessing time and \(\bigO(mn^{1-q})\) query time, contradicting \Cref{pro:preprocessed-gap-bcp-hardness}.
    To that end, let a threshold \(t \in [0,d]\) and a set \(\sB=\{\vb_1,\ldots,\vb_n\} \subseteq \sM_d\) be an input to \(\PreGapBCP(n,m,d,\varepsilon)\), together with queries \(\va_1,\ldots,\va_m \in \sM_d\) revealed after preprocessing.
    Applying \Cref{lem:small-t-bcp} to \(C_d\) and \(q\) gives a constant \(c\in(0,C_d/2)\) so that we can decide the \(\PreGapBCP\) instance with \(\bigO(nd)\le\bigO(n^k)\) preprocessing time and \(\bigO(mn^{1-q})\) decoding time if \(t\le c\log n\).
    We may there assume that \(t>c\log n\).
    We may further assume that \((1+\varepsilon)t\le d\).
    Indeed, if \((1 + \varepsilon) t > d\), then every conceivable pair \((\va,\vb) \in \{0,1\}^d \times \sB\) has Hamming distance at most \(d < (1 + \varepsilon) t\), so the problem is trivial to decide.

    We construct a KV cache \((\mK,\mV)\) and a query \(\mQ\) such that the \(\PreGapBCP\) instance can be decided from an attention approximation after KV cache preprocessing.
    After seeing only \(t\) and \(\sB\), we apply \Cref{lem:gadget} with parameters \(\varepsilon\), \(c\) and choose \(t\) as the threshold and \(\sB\) as the set.
    This supplies the constant \(C_B>0\) and allows us to construct, during preprocessing and in \(\bigO(ND) = \bigO(n \log n) \le \bigO(n^k)\) time, matrices \(\mK \in [-B,B]^{N \times D}\) and \(\mV \in [0,1]^{N \times D}\), which we preprocess with the hypothetical preprocessing algorithm for attention in \(\bigO(n^k)\) time.

    After that, let the vectors \(\va_1,\ldots,\va_m\) be revealed.
    Now, \Cref{lem:gadget} allows us to construct, in \(\bigO(mD) = \bigO(m\log n) \le \bigO(mn^{1-q})\) time, a matrix \(\mQ \in [-B,B]^{m \times D}\) such that \(\mY \coloneqq \Att(\mQ,\mK,\mV) = \vy\ve_1^\top\) for some \(\vy \in \R^m\) with
    \begin{equation} \label{eq:preprocessed-reduction:yes}
        y_i
        \ge
        \frac{n}{n+1}
        \qquad
        \text{if } \min_{j\in[n]}\|\va_i-\vb_j\|_0 \le t
    \end{equation}
    and
    \begin{equation} \label{eq:preprocessed-reduction:no}
        y_i
        \le
        \frac{1}{n+1}
        \qquad
        \text{if } \min_{j\in[n]}\|\va_i-\vb_j\|_0 \ge (1+\varepsilon)t
    \end{equation}
    for all \(i \in [m]\).
    We use that to decide the \(\PreGapBCP\) instance.
    To that end, run the assumed approximate attention decoding algorithm on \(\mQ\) and let \(\mT \in \R^{m \times D}\) be its output.
    
    Suppose first that the algorithm guarantees \(\|\mT - \mY\|_{\ell_\infty} \le \eta_\mathrm{abs}\).
    If \(\min_{i\in[m],j\in[n]} \|\va_i-\vb_j\|_0 \le t\), then
    \[
        T_{i,1}
        \ge
        Y_{i,1} - \eta_\mathrm{abs}
        \overset{\eqref{eq:preprocessed-reduction:yes}}{\ge}
        \frac{n}{n+1} - \eta_\mathrm{abs}
        >
        \frac{1}{2}
    \]
    for some \(i \in [m]\).
    If, instead, \(\min_{i\in[m],j\in[n]} \|\va_i-\vb_j\|_0 \ge (1+\varepsilon)t\), then
    \[
        T_{i,1}
        \le
        Y_{i,1} + \eta_\mathrm{abs}
        \overset{\eqref{eq:preprocessed-reduction:no}}{\le}
        \frac{1}{n+1} + \eta_\mathrm{abs}
        <
        \frac{1}{2}
    \]
    for all \(i \in [m]\).
    Thus, we can decide the \(\PreGapBCP\) instance by checking if any entry of \(\mT\) exceeds \(1/2\).

    Suppose now that, instead, the algorithm guarantees \(\|\mT-\mY\|_{\Schatten_p} \le \eta_\mathrm{rel} \|\mY\|_{\Schatten_p}\).
    Since \(\mY=\vy\ve_1^\top\), \(\|\mY\|_{\Schatten_p} = \|\vy\|_2 \le \sqrt{m}\).
    If \(\min_{i\in[m],j\in[n]} \|\va_i-\vb_j\|_0 \le t\), then
    \[
        \|\mT\|_{\Schatten_p}
        \ge
        \|\mY\|_{\Schatten_p}
        -
        \|\mT-\mY\|_{\Schatten_p}
        \ge
        (1-\eta_\mathrm{rel}) \|\mY\|_{\Schatten_p}
        =
        (1-\eta_\mathrm{rel}) \|\vy\|_2
        \overset{\eqref{eq:preprocessed-reduction:yes}}{\ge}
        (1-\eta_\mathrm{rel})\frac{n}{n+1}.
    \]
    for some \(i \in [m]\).
    If, instead, \(\min_{i\in[m],j\in[n]} \|\va_i-\vb_j\|_0 \ge (1+\varepsilon)t\), then
    \[
        \|\mT\|_{\Schatten_p}
        \le
        \|\mY\|_{\Schatten_p}
        +
        \|\mT-\mY\|_{\Schatten_p}
        \le
        (1+\eta_\mathrm{rel}) \|\mY\|_{\Schatten_p}
        =
        (1+\eta_\mathrm{rel}) \|\vy\|_2
        \overset{\eqref{eq:preprocessed-reduction:no}}{\le}
        (1+\eta_\mathrm{rel})\frac{\sqrt{m}}{n+1}.
    \]
    Thus, we can decide the \(\PreGapBCP\) instance by checking if \(\|\mT\|_{\Schatten_p} > (1-\eta_\mathrm{rel})\frac{n}{n+1}\).
    Computing \(\|\mT\|_{\Schatten_p}\) takes \(\bigO(mD^2)=\bigO(m\log^2n)\le\bigO(mn^{1-q})\) time.
    Either guarantee gives an algorithm for deciding the \(\PreGapBCP\) instance from \(\mT\) in \(\bigO(D) = \bigO(\log n) \le \bigO(n^{1-q})\) time.
    Computing \(\mT\), by assumption, only requires \(\bigO(n^k)\) preprocessing time and \(\bigO(mn^{1-q})\) decoding time.
    That contradicts \Cref{pro:preprocessed-gap-bcp-hardness} and thereby proves the claim.
\end{proof}

\getkeytheorem{thm:constant-error-hardness}
\begin{proof}
    This follows directly from \Cref{pro:preprocessed-batched-attention-hardness} with \(m(n) \coloneqq n\).
\end{proof}

\getkeytheorem{thm:preprocessing-hardness}
\begin{proof}
    This follows directly from \Cref{pro:preprocessed-batched-attention-hardness} with \(m(n) \coloneqq 1\).
\end{proof}

\subsection{Finding close pairs from a sparse attention approximation} \label{subsec:proofs:encoding-closest-pairs-with-sparse-attention}
We next show how a sparse attention algorithm can be used to identify approximate closest pairs. As before, we encode distances between a query vector and the preprocessed set as attention scores, with closer vectors receiving larger scores. We would then like to recover an approximate closest pair from any small set of keys carrying a constant fraction of the attention mass. For this argument to apply, however, the resulting attention instance must satisfy the sparsity promise of \Cref{prb:sparse-attention}. But a closest pair instance may contain many approximately closest vectors, causing the corresponding keys to receive similar scores and the attention mass to be spread among them. We overcome this by randomly perturbing the scores so that, with sufficiently high probability, the score corresponding to a close pair is amplified enough to dominate the attention. Exponential random variables provide exactly this effect. At the same time, the random perturbation must not allow a far pair to dominate by chance. We therefore truncate the exponential variables, limiting how much any score can be amplified while still allowing a close pair to separate from the rest. The following lemma shows that these truncated random variables indeed still create the required separation. \Cref{lem:sparse-gadget} then translates this separation into an attention instance in which, with high probability, an approximate closest pair carries almost all of the attention mass and can therefore be identified by a retrieval algorithm for sparse attention.

\begin{lemma} \label{lem:exponential-separation}
    Fix \(L>0\), \(n\ge2\), and \(z_1,\ldots,z_n\in\R\).
    Sample exponential random variables \(E_1,\ldots,E_n \overset{\mathrm{iid}}{\sim} \Exp(\theta)\) of rate \(\theta>0\) and clip them to \(\widehat E_j \coloneqq \min\{E_j,L\}\).
    Then, for any \(g>0\),
    \[
        \Pr\Bigl[ \exists j^\star\in[n] : \forall j^\star \neq j\in[n] : (s_{j^*} + \widehat E_{j^\star}) - (s_j + \widehat E_j) \ge g \Bigr]
        \ge e^{-\theta g} - ne^{-\theta L}.
    \]
\end{lemma}
\begin{proof}
    Define the scores \(\widehat Z_j \coloneqq s_j + \widehat E_j\) and their corresponding values before clipping \(Z_j \coloneqq s_j + E_j\).
    Fix an arbitrary \(j^\star\in[n]\) and condition on \(E_{-j^\star}\), that is, all \(E_j\) except \(E_{j^\star}\).
    Among these, let
    \[
        M_{j^\star} \coloneqq \max_{j\in[n], \ j\neq j^\star} Z_j
    \]
    be the maximum score.
    Then,
    \begin{align*}
        \Pr\bigl[
            Z_{j^\star} \ge M_{j^\star} + g \mid E_{-j^\star}
        \bigr]
        &=
        e^{-\theta\max\{0,M_{j^\star}+g-s_{j^\star}\}}
        \\&\ge
        e^{-\theta g}e^{-\theta \max\{0,M_{j^\star}-s_{j^\star}\}}
        =
        e^{-\theta g} \Pr\bigl[
            Z_{j^\star} \ge M_{j^\star} \mid E_{-j^\star}
        \bigr].
    \end{align*}
    Taking expectations over \(E_{-j^\star}\),
    \[
        \Pr\bigl[
            Z_{j^\star} \ge M_{j^\star} + g
        \bigr]
        \ge
        e^{-\theta g} \Pr\bigl[
            Z_{j^\star} \ge M_{j^\star}
        \bigr].
    \]
    Summing over all possible choices of \(j^\star\) and using the fact that the continuous random variables \(s_j+E_j\) have a unique maximum almost surely,
    \begin{align*}
        \Pr\Bigl[
            \exists j^\star\in[n] : \forall j^\star \neq j\in[n] :
            Z_{j^\star} - Z_j \ge g
        \Bigr]
        &=
        \sum_{j^\star\in[n]} \Pr\bigl[
            Z_{j^\star} \ge M_{j^\star} + g
        \bigr]
        \\&\ge
        e^{-\theta g} \sum_{j^\star\in[n]} \Pr\bigl[
            Z_{j^\star} \ge M_{j^\star}
        \bigr]
        \ge
        e^{-\theta g}.
    \end{align*}
    Finally, it may happen that some of the exponential random variables are clipped. Therefore, working on the even that no clipping occurs and using the union bound,
    \begin{align*}
        &\phantom{\ge} \Pr\Bigl[
            \exists j^\star\in[n] : \forall j^\star \neq j\in[n] :
            \widehat Z_{j^\star} - \widehat Z_j \ge g
        \Bigr]
        \\&\ge
        \Pr\Bigl[
            \exists j^\star\in[n] : \forall j^\star \neq j\in[n] :
            Z_{j^\star} - Z_j \ge g
        \Bigr]
        -
        \Pr\Bigl[
            \exists j\in[n] : E_j>L
        \Bigr]
        \\&\ge
        e^{-\theta g} - ne^{-\theta L}.
        \tag*{\qedhere}
    \end{align*}
\end{proof}

\begin{lemma} \label{lem:sparse-gadget}
    Fix \(\varepsilon,c,C_d > 0\) and \(\gamma\in(0,1)\).
    There exists \(C_B>0\) such that the following holds for all sufficiently large \(n \in \sN\), \(d \coloneqq C_d \log n\), and \(B \coloneqq C_B \sqrt{\log n}\).
    
    Choose \(t \ge c \log n\) with \(d \ge (1+\varepsilon)t\), a set of vectors \(\sB = \{\vb_1,\ldots,\vb_n\} \subseteq \sM_d\), and a query \(\va\in\sM_d\).
    One can sample matrices \(\mK\in[-B,B]^{n\times2d}\) independent of \(\va\) and in \(\bigO(nd)\) time, and construct a vector \(\vq\in[-B,B]^{2d}\) in \(\bigO(d)\) time such that
    \[
        \Pr\Bigl[
            \exists j^\star\in[n]:
            \|\va-\vb_{j^\star}\|_0 \le (1+\varepsilon)t
            \;\text{ and }\;
            \softmax_{j^\star}\bigl(\mK\vq/(2d)\bigr)>1-\frac{1}{n}
        \Bigr]
        \ge
        \frac{1}{2} n^{-\gamma}
    \]
    if there exists \(\vb^\star\in\sB\) with \(\|\va-\vb^\star\| \le t\).
\end{lemma}
\begin{proof}
    Define
    \[
        \theta \coloneqq \frac{\gamma}{2},\qquad
        R \coloneqq \frac{3}{\theta},\qquad\text{and}\qquad
        \beta \coloneqq \frac{4d(R+1)\log n}{\varepsilon t}.
    \]
    We first sample clipped exponential random variables
    \[
        \widehat E_j\coloneqq\min\{E_j,R\log n\},
        \qquad\text{where}\quad
        E_1,\ldots,E_n \overset{\mathrm{iid}}{\sim} \Exp(\theta).
    \]
    We then construct the key matrix and query vector as
    \[
        \mK\coloneqq\sqrt{\beta} \begin{bmatrix}
            \vb_1^\top & \frac{2\widehat E_1}{\beta}\vone_d^\top \\
            \vdots & \vdots \\
            \vb_n^\top & \frac{2\widehat E_n}{\beta}\vone_d^\top \\
        \end{bmatrix}
        \in\R^{n\times2d}
        \qquad\text{and}\qquad
        \vq\coloneqq\sqrt{\beta} \begin{bmatrix}
            \va^\top & \vone_d^\top
        \end{bmatrix}
        \in\R^{2d}.
    \]
    Next, choose \(C_B\) so that the constructed matrices satisfy the required entry bounds.
    Since \(d=C_d\log n\) and \(t \ge c\log n\),
    \[
        \beta
        =
        \frac{4d(R+1)\log n}{\varepsilon t}
        \le
        \frac{4C_d(R+1)\log n}{\varepsilon c}.
    \]
    Moreover, since \(t \le d/(1+\varepsilon)\) and \(\widehat E_j \le R\log n\),
    \[
        0
        \le
        \frac{2\widehat E_j}{\beta}
        \le
        \frac{R\varepsilon}{2(R+1)(1+\varepsilon)}
        <
        \frac{1}{2}.
    \]
    Consequently, there is a constant \(C_B\) depending only on \(C_d\), \(\varepsilon\), \(c\), and \(R\) such that \(\|\mK\|_{\ell_\infty} \le C_B\sqrt{\log n}\) and \(\|\vq\|_{\ell_\infty} \le C_B\sqrt{\log n}\).
    
    By construction, the logit between the query and key \(j\in[n]\) is
    \begin{align} \label{eq:sparse-gadget:logit}
        \widehat Z_j
        &\coloneqq
        \frac{\mK_j^\top\vq}{2d}
        =
        \frac{\beta}{2d}\Bigl( \langle\va,\vb_j\rangle + \frac{2d\widehat E_j}{\beta} \Bigr)
        =
        \frac{\beta}{2d} \langle\va,\vb_j\rangle + \widehat E_j \notag
        \\&=
        \frac{\beta}{4d} \bigl( \|\va\|_2^2+\|\vb_j\|_2^2-\|\va-\vb_j\|_0 \bigr) + \widehat E_j \notag
        =
        \frac{\beta}{4d} \bigl( d-\|\va-\vb_j\|_0 \bigr) + \widehat E_j
        \\&=
        \frac{\beta}{4} \Bigl( 1-\frac{\|\va-\vb_j\|_0}{d} \Bigr) + \widehat E_j,
    \end{align}
    where we use that \(\|\va\|_0=\|\vb\|_0=d/2\).
    Suppose now there is some close partner for \(\va\) at index \(i\in[n]\) with \(\|\va-\vb_i\|_0 \le t\).
    Then, for all far vectors at index \(j\in[n]\) with \(\|\va-\vb_j\|_0 > (1+\varepsilon)t\),
    \begin{align*}
        \widehat Z_i-\widehat Z_j
        &\overset{\eqref{eq:sparse-gadget:logit}}{=}
        \frac{\beta}{4d} \bigl( \|\va-\vb_j\|_0-\|\va-\vb_i\|_0 \bigr) + \widehat E_i-\widehat E_j
        \\&\ge
        \frac{\beta\varepsilon t}{4d} + \widehat E_i-\widehat E_j
        \\&\ge
        (R+1)\log n - R\log n
        >
        0,
    \end{align*}
    so logits of close partners are separated from those of far partners. In particular,
    \begin{equation} \label{eq:sparse-reduction:maximizer}
        j^\star
        \coloneqq
        \argmax_{j\in[n]} \widehat Z_j
        \in
        \bigl\{
            j\in[n] : \|\va-\vb_j\|_0 \le (1+\varepsilon)t
        \bigr\}.
    \end{equation}

    Now, compare this maximal logit at position \(j^\star\) against the other logits.
    Applying \Cref{lem:exponential-separation} with \(\theta=\gamma/2\), \(L=R\log n\), and \(g=2\log n\) gives
    \[
        \Pr\Bigl[ \forall j^\star \neq j\in[n] : \widehat Z_{j^\star} - \widehat Z_j \ge 2\log n \Bigr]
        \ge n^{-\gamma} - n^{-2}
        \ge \frac{1}{2} n^{-\gamma}
    \]
    for all sufficiently large \(n\).
    On the same event,
    \[
        1-\softmax_{j^\star}(\vq\mK^\top/(2d))
        =
        \frac{\sum_{j^\star \neq j\in [n]} e^{\widehat Z_j}}{e^{\widehat Z_j} + \sum_{j^\star \neq j\in [n]} e^{\widehat Z_j}}
        \le
        \frac{(n-1) e^{\widehat Z_j - 2\log n}}{e^{\widehat Z_j}}
        =
        \frac{n-1}{n^2}
        \le
        \frac{1}{n},
    \]
    which proves the claim.
\end{proof}

\subsection[Proof of Theorem~\ref*{thm:sparse-hardness}]{Proof of \Cref{thm:sparse-hardness}} \label{subsec:proofs:sparse-attention}
We can now prove \Cref{thm:sparse-hardness} by applying \Cref{lem:sparse-gadget} to the closest pair problem from \Cref{def:preprocessed-gap-bcp}.
The construction identifies an approximate closest pair by letting its key receive almost all the attention for a given query, with some probability decreasing polynomially in the input size \(n\). To ensure a large success probability, we construct and preprocess sufficiently many independent copies, which increases preprocessing only polynomially and can be chosen so that the total decoding time remains truly sublinear. On a successful copy, every valid sparse-attention output must contain the dominant key.

\getkeytheorem{thm:sparse-hardness}
\begin{proof}
    Assume w.l.o.g. that \(k>1\) and \(q \in (0,1/2)\).
    We choose \(C_d,\varepsilon>0\) as the constant from \Cref{pro:preprocessed-gap-bcp-hardness} so that \(\PreGapBCP(n,1,d,\varepsilon)\) cannot be solved with \(\bigO(n^{k+1})\) preprocessing time and \(\bigO(n^{1-q/2})\) decoding time for \(d\coloneqq \frac{C_d}{2}\log n\).
    Suppose, for contradiction, that \(\SparseAtt(n,D,B,\alpha)\) can be solved with \(\bigO(n^k)\) preprocessing time and \(\bigO(n^{1-q})\) decoding time.

    We construct an algorithm for \(\PreGapBCP(n,1,d,\varepsilon)\) with \(\bigO(n^{k+1})\) preprocessing time and \(\bigO(n^{1-q/2})\) query time, contradicting \Cref{pro:preprocessed-gap-bcp-hardness}.
    To that end, let a threshold \(t\in[0,d]\) and a set \(\sB=\{\vb_1,\ldots,\vb_n\} \subseteq \sM_d\) be an input to \(\PreGapBCP(n,1,d,\varepsilon)\), together with a query \(\va\in\sM_d\) revealed after preprocessing.
    Applying \Cref{lem:small-t-bcp} to \(C_d\) and \(q\) gives a constant \(c\in(0,C_d/2)\) so that we can decide the \(\PreGapBCP\) instance with \(\bigO(nd)\le\bigO(n^k)\) preprocessing time and \(\bigO(n^{1-q})\) decoding time if \(t\le c\log n\).
    We may there assume that \(t>c\log n\).
    We may further assume that \((1+\varepsilon)t\le d\).
    Indeed, if \((1 + \varepsilon) t > d\), then every conceivable pair \((\va,\vb) \in \{0,1\}^d \times \sB\) has Hamming distance at most \(d < (1 + \varepsilon) t\), so the problem is trivial to decide.

    We show that the \(\PreGapBCP\) instance can be decided by finding a small set of keys in \(\mK\) receiving at least an \(\alpha\)-fraction of the attention for the query \(\vq\).
    After seeing only \(t\) and \(\sB\), we apply \Cref{lem:sparse-gadget} with parameters \(\varepsilon\), \(c\), \(C_d\), and \(\gamma=q/4\).
    This supplies the constant \(C_B>0\) and allows us to sample, during preprocessing and in \(\bigO(Tnd) \le \bigO(n^{k+1})\) time, \(T\coloneqq C_Tn^\gamma\) independent key matrices \(\mK^{(1)},\ldots,\mK^{(T)} \in [-B,B]^{n\times D}\), which we preprocess with the hypothetical preprocessing algorithm for sparse attention in \(\bigO(Tn^k) \le \bigO(n^{k+1})\) time.

    After that, let the vector \(\va\in\sM_d\) be revealed.
    Now, \Cref{lem:sparse-gadget} allows us to construct, in \(\bigO(Td) \le \bigO(n^{1-q})\) time, the queries \(\vq^{(1)},\ldots,\vq^{(T)} \in [-B,B]^D\) corresponding to the key matrices and guaranteeing that
    \begin{equation} \label{eq:sparse-hardness:success-probability}
        \Pr\Bigl[
            \exists j^\star\in[n]:
            \|\va-\vb_{j^\star}\|_0 \le (1+\varepsilon)t
            \;\text{ and }\;
            \softmax_{j^\star}\bigl(\mK^{(\ell)}\vq^{(\ell)}/D\bigr)>1-\frac{1}{n}
        \Bigr]
        \ge
        \frac{1}{2} n^{-\gamma}
    \end{equation}
    for each copy \(\ell\in[T]\) if \(\min_{j\in[n]}\|\va-\vb_j\|_0 \le t\).
    Again for each copy, we run the hypothetical decoding algorithm for sparse attention, resulting in explicit sets \(\sI^{(1)},\ldots,\sI^{(T)} \subseteq [n]\).
    Each invocation runs in \(\bigO(n^{1-q})\) time by assumption, so, in particular, \(|\sI^{(\ell)}| = \bigO(n^{1-q})\) for each \(\ell\in[T]\).
    For each copy \(\ell\in[T]\) and each index \(i\in\sI^{(\ell)}\), we check if \(\|\va-\vb_i\|_0 \le (1+\varepsilon)t\).
    If we find such an index \(i\), we accept the \(\PreGapBCP\) instance.
    Otherwise, reject the instance.

    The algorithm never falsely accepts an instance.
    If, however, \(\min_{j\in[n]}\|\va-\vb_j\|_0 \le t\), then it follows from \eqref{eq:sparse-hardness:success-probability} that only with probability
    \[
        \Bigl(1-\frac{1}{2}n^{-\gamma}\Bigr)^{C_T n^{\gamma}}
        \le
        e^{-C_T/2}
    \]
    no copy contains an index receiving more than \(1-\frac{1}{n}\) attention.
    Conversely, we can choose \(C_T\) as a sufficiently large absolute constant so that, with probability at least \(2/3\), some copy \(\ell\in[T]\) has an index \(i\in[n]\) with
    \[
        \softmax_i\bigl(\mK^{(\ell)}\vq^{(\ell)}/D)>1-\frac{1}{n}.
    \]
    On such a copy, assuming \(n\) is sufficiently large, \(\sI^{(\ell)}\) can only receive at least \(\alpha\) attention if it includes the index \(i\).
    Consequently, the hypothetical decoding algorithm for sparse attention guarantees \(i\in\sI^{(\ell)}\), so our algorithm for \(\PreGapBCP\) accepts.

    Put together, we obtain a bounded-error algorithm for \(\PreGapBCP(n,1,d,\varepsilon)\).
    The algorithm uses \(\bigO(n^{k+1})\) time during preprocessing and \(\bigO(n^\gamma n^{1-q}d) \le \bigO(n^{1-q/2})\) time during decoding.
    That contradicts \Cref{pro:preprocessed-gap-bcp-hardness} and thereby proves the claim.
\end{proof}

\end{document}